\documentclass[journal]{IEEEtran}  % Comment this line out
\IEEEoverridecommandlockouts
\usepackage{changes}
\usepackage{adjustbox}
\usepackage{caption}
\usepackage{cite}
\usepackage{amsmath,amssymb,amsfonts}
\usepackage{graphicx}
\usepackage[linesnumbered,algoruled,boxed,lined]{algorithm2e}
\usepackage{textcomp}
\usepackage{comment}
\usepackage{xcolor}
\usepackage{paralist}
\def\BibTeX{{\rm B\kern-.05em{\sc i\kern-.025em b}\kern-.08em
    T\kern-.1667em\lower.7ex\hbox{E}\kern-.125emX}}
\usepackage{times}  %Required
\usepackage{helvet}  %Required
\usepackage{todonotes}
\usepackage{amsmath}
\usepackage{amssymb}
 
\usepackage{amsthm}
\usepackage{mathtools}
\usepackage{acronym}
\usepackage{courier}  %Required
\newif\ifuseboldmathops
\newif\ifuseittextabbrevs
\useboldmathopstrue   % comment out to use mathbb
\ifuseittextabbrevs

	\newcommand{\etal}{{et~al.}}
\else

	\newcommand{\etal}{et~al.}
\fi

\ifuseboldmathops
	\newcommand{\reals}{\mathbf{R}}

\else
	\newcommand{\reals}{\mathbb{R}}

\fi

\ifuseboldmathops

\else

\fi

\ifuseboldmathops

\else

\fi

\ifuseboldmathops

\else

\fi

\newcommand{\norm}[1]{\lVert#1\rVert}

\renewcommand{\vec}[1]{\mathbf{#1}}

\acrodef{mdp}[MDP]{Markov Decision Process}
\acrodef{pomdp}[POMDP]{Partially Observable Markov Decision Process}

\theoremstyle{definition}

\newtheorem{proposition}{Proposition}
\newtheorem{remark}{Remark}

\acrodef{smdp}[Semi-MDP]{Semi-Markov decision process}
\acrodef{rl}[RL]{reinforcement learning}
\acrodef{mcts}[MCTS]{Monte Carlo tree search}
\acrodef{uct}[UCT]{Upper Confidence Bound 1 applied to trees}
\acrodef{scltl}[scLTL]{syntactically co-safe LTL}
\acrodef{ssp}[SSP]{Stochastic Shortest Path}
\acrodef{p2sg}[SG(2)]{Two-player Stochastic Game}
\acrodef{dof}[DOF]{degree of freedom}
\acrodef{cpg}[CPG]{Central Pattern Generator}
\acrodef{nn}[NN]{Neural Network}
\acrodef{snn}[SNN]{Spiking Neural Net}
\acrodef{rstdp}[R-STDP]{Reward-Modulated Spike-Timing-Dependent Plasticity}
\acrodef{gp}[GP]{Genetic Programming}
\acrodef{ppoc}[PPOC]{Proximal Policy Optimization Option-Critics}
\acrodef{dr}[DR]{Domain Randomization}
\acrodef{bibo}[BIBO]{Bounded-input, Bounded-Output}

\acrodefplural{dof}[DOFs]{degrees of freedom}
\acrodefplural{cpg}[CPGs]{Central Pattern Generators}

\usepackage{url}  %Required
\usepackage{graphicx}  %Required
\usepackage{bm}
\usepackage{color}
\usepackage{amsfonts}       % blackboard math symbols
\usepackage{nicefrac}       % compact symbols for 1/2, etc.
\usepackage{microtype}      % microtypography
\usepackage{lipsum}
\usepackage{amsmath}
\usepackage{subcaption}
\usepackage{algpseudocode} 
\usepackage{tabularx}

\definecolor{ultramarine}{rgb}{0.07, 0.04, 0.56}

\SetKwInput{KwInput}{Input}
\SetKwInput{KwOutput}{Output}

\title{\LARGE \bf  Integrating Contact-aware Feedback CPG System for Learning-based Soft Snake Robot Locomotion Controllers }

\author{Xuan Liu\textsuperscript{1}, Cagdas D. Onal\textsuperscript{1}, and Jie Fu\textsuperscript{2,*}
\thanks{\textsuperscript{1} Xuan Liu and Cagdas Onal are with the Robotics Engineering Department at Worcester Polytechnic Institute, Worcester, MA, US. \textsuperscript{2} Jie Fu is with the Department of Electrical and Computer Engineering, University of Florida, Gainesville, FL, US.
\textsuperscript{*} The corresponding author.} 
\thanks{\tt\small xliu9, cdonal, @wpi.edu, fujie@ufl.edu}%
\thanks{This work was supported in part by the National Science Foundation under grant \#1728412.}%
}

\begin{document}

\maketitle
\thispagestyle{empty}
\pagestyle{empty}

\begin{abstract}

%Integrating contact-awareness into a soft snake robot and efficiently controlling its locomotion in response to contact information present significant challenges. 
This paper aims to solve the contact-aware locomotion problem of a soft snake robot by developing bio-inspired contact-aware locomotion controllers. To provide effective contact information for the controllers, we develop a scale-covered sensor structure mimicking natural snakes' \textit{scale sensilla}. In the design of the control framework, our core contribution is the development of a novel sensory feedback mechanism for the Matsuoka central pattern generator (CPG) network. This mechanism allows the Matsuoka CPG system to work like a ``spinal cord'' in the whole contact-aware control scheme, which simultaneously takes the stimuli including tonic input signals from the ``brain'' (a goal-tracking locomotion controller) and sensory feedback signals from the ``reflex arc'' (the contact reactive controller), and generates rhythmic signals to actuate the soft snake robot to slither through densely allocated obstacles. In the ``reflex arc'' design, we develop two distinctive types of reactive controllers -- 1) a reinforcement learning (RL) sensor regulator that learns to manipulate the sensory feedback inputs of the CPG system, and 2) a local reflexive sensor-CPG network that directly connects sensor readings and the CPG's feedback inputs in a specific topology. Combining with the locomotion controller and the Matsuoka CPG system, these two reactive controllers facilitate two different contact-aware locomotion control schemes. The two control schemes are tested and evaluated in both simulated and real soft snake robots, showing promising performance in the contact-aware locomotion tasks. The experimental results also validate the advantages of the modified Matsuoka CPG system with a new sensory feedback mechanism for bio-inspired robot controller design.

\end{abstract}

\section{Introduction}
Soft continuum robots have unique advantages in traversing through cluttered and confined environments, due to their flexible body structure and deformable materials. Applications of soft continuum robots in contact-aware environments include search-and-rescue \cite{hawkes2017soft}, pipe inspection\cite{majidi2014soft}, and medical surgery \cite{wang2017cable}. In particular, soft robotic snakes have the unique potential that any part of their body, if properly controlled, could adapt to and reduce the impact from collisions, or even benefit from the propulsion force generated by the contacts with obstacles. In this paper, we investigate the following two questions: 
\begin{itemize}
    \item How to design the contact-aware control method for the soft snake robot that can intelligently adapt to and employ the contact force from densely distributed obstacles during locomotion?
    \item How to design the sensory and body structure on a snake-like soft robot to make the tactile perception more sensitive while avoiding jamming, and making the contact-aware locomotion more energy efficient? 
\end{itemize}

In literature, several research groups have studied this unique topic in snake robot locomotion. The solutions to the contact-aware locomotion \cite{7838565, Kano2017, 2023classicfeedback} are mainly studied and implemented on rigid snake robots and most of the control methods are model-based. Transeth \etal \cite{Transeth2007, transeth2008snake} first defined this property as the \textit{obstacle-aided locomotion},  wherein the snake robot actively employs external objects to generate propulsion forces during the locomotion. Their pioneer work proposed a two-module framework of obstacle-aided locomotion that includes: (a) a path planner that searches for a trajectory with more active contact chance for the rigid snake robot, and (b) a motion controller that controls the snake robot's real-time body movements to optimally utilize the contacts between the robot and the environment and generate desired propulsion force for the locomotion. In \cite{liljeback2010hybrid, liljeback2011experimental, GRAVDAHL2022247}, a hybrid controller is developed, where a contact event is treated individually by a reactive controller that maximizes the total propulsion force at the contacting moment. This controller has been applied to a rigid snake robot and showed its reliability in maintaining beneficial propulsion force. Kano \etal \cite{kano2012local, kano2013} proposed local reflexive mechanisms that interrogate the contact status between the snake robot and the obstacles to determine whether the contact is beneficial to the locomotion. In this approach, only a segment of the robot links neighboring to the link in contact react to the sensory feedback. On the basis of the local reflexive control method, a Tegotae heuristic scoring function is established by \cite{10.1093/icb/icaa014, Kano2017, kano2019}, for selecting which kind of reaction should be applied to the contacting link of the robot given certain situations including the snake robot's shape and contacting part of the robot. 
 From the bio-inspired perspective, inspired by the entrainment properties of the neural oscillators that allow the systems' output to be synchronized with the sensory feedback, several studies \cite{ijspeert2021science, 2023classicfeedback} introduce CPG systems to the control loop of the snake robots to process the sensory feedback signals during locomotion. However, in most existing work the locomotion control inputs of the feedback CPG systems are usually constant or rather simple sinusoidal trajectories due to the difficulty of coordinating multiple complex signals through a CPG system. Conducting both intelligent locomotion control and sensory feedback control on the CPG-driven snake robot system is still a promising but rarely explored research topic.

% However, the soft body in \cite{kano2012local} is only approximated in simulation by increasing the degrees of freedom of each joint with multiple linear actuators. % Moreover, the design that has too many linear actuators coupled together in a snake-like scaffold also makes it difficult for manufacturing.
% \jie{Kano also studied rigid robot right?} \xuan{they claimed that the robot is soft-bodied with spring damper scaffold, but it is not rigorously soft, it is just a rigid robot with very high DoF, which has only been realized in a simulator (I doubt the feasibility of their design for manufacturing), their real world data was collected from a real snake, not a robot.}

%Although obstacle-aided locomotion is beneficial for biological snakes to move faster in a cluttered terrain, realizing such a capability in a robot introduces high demands on the mechanical design with actuators and sensors. 

% For example, a work shows that a simulated ribbon eel could traverse the cluttered environment with RL controller learned in the free space\cite{min2019softcon}, under the assumption of perfect actuation without any noise or delay.
 % However, the proposed model-free controller did not incorporate the contact information into the feedback

So far, the results on contact-aware control for soft snake robots are scarce. Although a few end-to-end soft robot controllers \cite{min2019softcon} perform well in simulation by assuming fully proprioceptive observations, it would be appealing to enable such a capability for soft snake robots in the real world. Moving from rigid snake robots to soft snake robots in contact-aware locomotion control faces many challenges: 

\begin{itemize}
    \item  Due to the continuum of the pneumatic actuators, it is infeasible to construct accurate kinematic/dynamic models for a soft snake robot, rendering model-based control ineffective or inapplicable.
    \item  The pneumatic actuators in soft snake robots have nonlinear, delayed, and stochastic dynamical response given inputs, making it difficult to achieve fast responses through model-based control compared to rigid snake robots.
    \item It is hard to embed tactile sensors in the soft material since the contact-free deformation of the soft body may interfere with the sensory data. As a result, the tactile sensors cannot be densely placed on the soft robot.
    \item Equipping tactile sensors could introduce more contact friction due to the material of the sensors, or cause more contact jamming due to the bumped shape of the sensors. 
    \item In the scenario when a soft snake robot is traversing among unknown obstacles, the tactile sensory inputs are usually discrete and unpredictable impulses, which can result in overshoot, latency and signal interference to a feedback control system. 
    % \item \color{blue} In the scenario when a soft snake robot is traversing unknown obstacles, the tactile sensory inputs are usually discrete and unpredictable impulses, which naturally brings overshoot, latency and signal interference issues to any feedback control system. \color{black}
\end{itemize}

% The facts that might harm the performance of obstacle aided locomotion on real robot including but not limited to: (1) The rigid robot snakes usually have much less number of contact sensors, and therefore the calculation of contact forces are not accurate; (2) The rigid body links are not as deformable as the soft body of real snakes, which causes its contact strategy to become less flexible; (3) Snake robots in the lab usually assume perfect knowledge about obstacles, while in nature the real snakes can react to the contact subtly with highly limited visual information; (4) Although the design of soft snake robot\cite{luo2014theoretical} has mitigated the deformation issue, the nonlinear, delayed and stochastic  dynamical response from the soft actuators makes it difficult to control the soft snake robot with model based methods. 

% \begin{figure}[h!]
%     \centering
%     \includegraphics[width=1\columnwidth]{figs/framework.png}
%     \caption{Schematic view of learning-based CPG controller.}
%     \label{fig:schematic_view}
%     % \vspace{-2ex}
% \end{figure}

\begin{figure}[h!]
\centering

\subfloat[]{
\includegraphics[width=0.9\columnwidth]{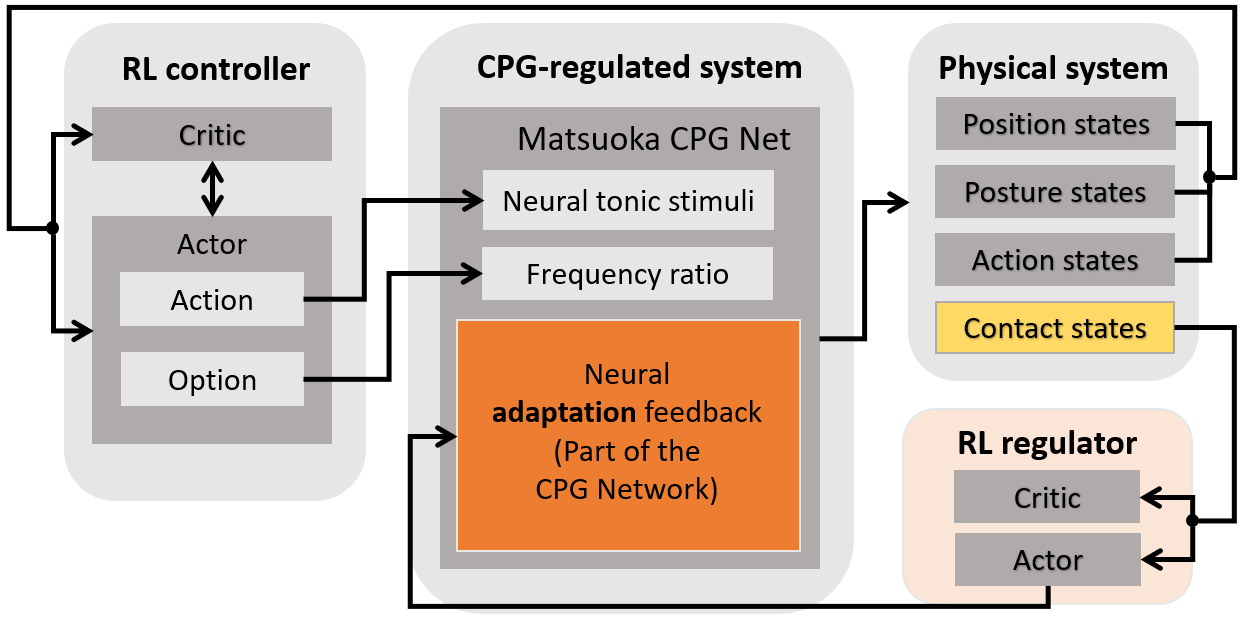}
\label{fig:learningreflex}}
\hfill
\subfloat[]{
\includegraphics[width=0.9\columnwidth]{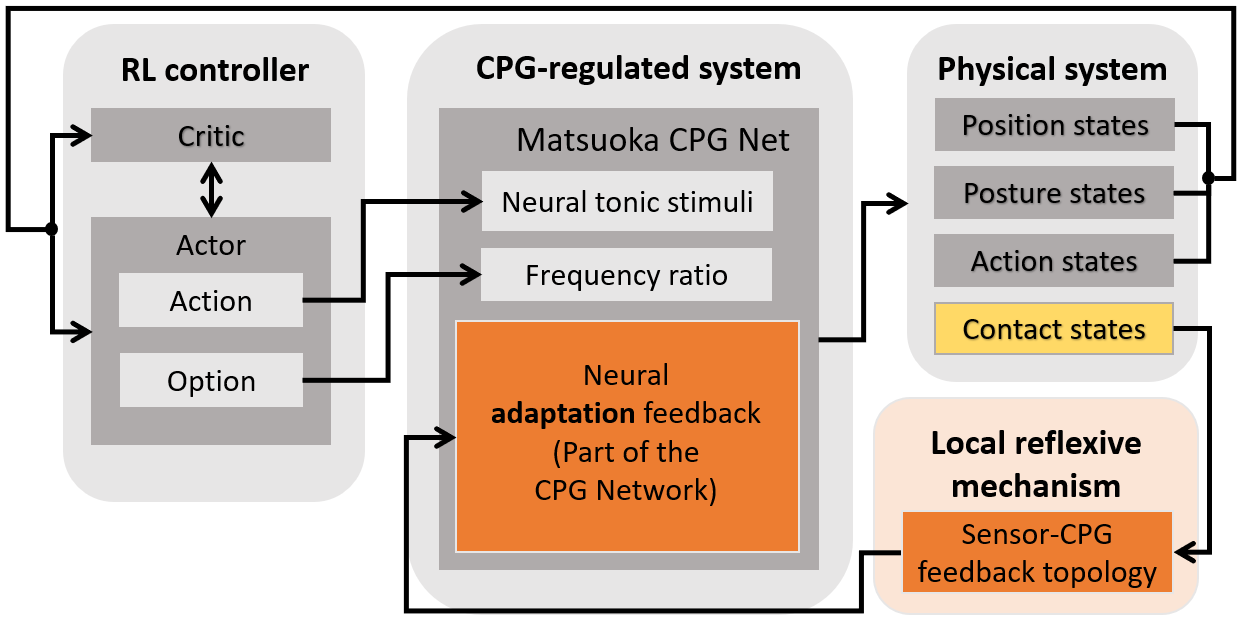}
\label{fig:localreflex}}
\caption{Schematic view of (a) learning reflexive PPOC-CPG, and (b) local reflexive PPOC-CPG controllers.}
\label{fig:schematic_view}
\end{figure}

% this controlled robot will be aware of the existence of obstacles in the vicinity and react correctly by taking obstacle-aided locomotion to achieve optimal goal tracking performance. This work is built on our previous work on soft snake robot \cite{xliu2020, xliu2023}, in which we established a model-free closed-loop goal tracking controller using central pattern generator (CPG) and reinforcement learning (RL). 

% the RL controller learns to control the tonic inputs to the CPG network, and the output of the CPG network will generate actuator inputs to the soft snake robot to achieve smooth and rhythmic locomotion in an obstacle-free environment. To enable reactive motions in a cluttered environment, we introduce 

% In this paper, we aim to develop a real-time feedback control

For the above problems, our solutions and contributions are summarized as follows:
\begin{itemize}
    \item \textbf{Tactile sensor design}: We design a group of magnetic-based tactile sensors (inspired by \cite{wang2016touchsensor}) with scale-like cover mimicking \textit{Scale Sensilla}\cite{crowe2016sensilla} on the scales of real animal snakes. This structure can cover larger sensing area with a small number of sensors. As a result, it improves the sensitivity to contacts with sparsely deployed sensors on the soft snake robot. In addition, the smoothness of the covering material and the scale-like structure reduce the contact friction and collision damage on the tactile sensors during contact-aware locomotion.
    \item \textbf{CPG system design}: We develop a novel feedback mechanism of the Matsuoka oscillator to process both the locomotion control signals and the tactile sensory feedback signals during the contact-aware locomotion of the soft snake robot. Through theoretical analysis, we leverage the unique advantages of the Matsuoka oscillator's feedback mechanism for reducing the overshoot, time latency, and interference of the feedback control signals.  
    \item \textbf{Design of contact reactive controllers}: Based on our modification of the Matsuoka oscillator, we designed two different contact reactive controllers (\emph{hybrid learning controller} in Fig. 1a and \emph{local reflexive controller} in Fig. 1b) for the modified Matsuoka CPG system to work along with the learning-based goal-tracking module developed in our previous work \cite{xliu2023}. These lead to two different control schemes for the contact-aware locomotion of the soft snake robot under the obstacle-based goal-tracking tasks. Each method has its own unique disadvantages and advantages: the learning-based reactive controller is computationally expensive but can iteratively learn to improve its performance. The local reflexive method is light-weighted and more robust because of its local reactive property but is heuristic with fixed policy.
    \item \textbf{Experimental implementations and validations}: The efficacy of the two control schemes is evaluated both in simulation and real world. We applied the two control schemes to our soft snake robot equipped with specially designed contact sensors, and finally realized the promising locomotion performance for traversing densely allocated obstacles, while maintaining performance on sharp turns and target reaching in both obstacle-free and obstacle-based environments.
\end{itemize}

The paper is organized as follows. First, we provide an overview of our schematic design of a soft snake robot equipped with contact sensors in Section~\ref{sec:sys}. In Section~\ref{sec:cpg}, we present the novel feedback mechanism of the Matsuoka CPG system with theoretical analyses and proofs to show its advantages in leveraging sensory feedback signals for obstacle-aware locomotion control. Based on this, we design two sensory reactive controllers in Section~\ref{sec:main}. In Section~\ref{sec:experiment}, we design several experiments to showcase the improvement of the snake robot's locomotion performance in obstacle-based environments under closed-loop control. Section~\ref{sec:conclusion} concludes and discusses future work.

\section{Hardware Design for Contact-aware Soft Snake Robot Locomotion}
\label{sec:sys}

\subsection{Design of a Tactile Sensor}

 \begin{figure}[h!]
    \centering
    \includegraphics[width=0.9\columnwidth]{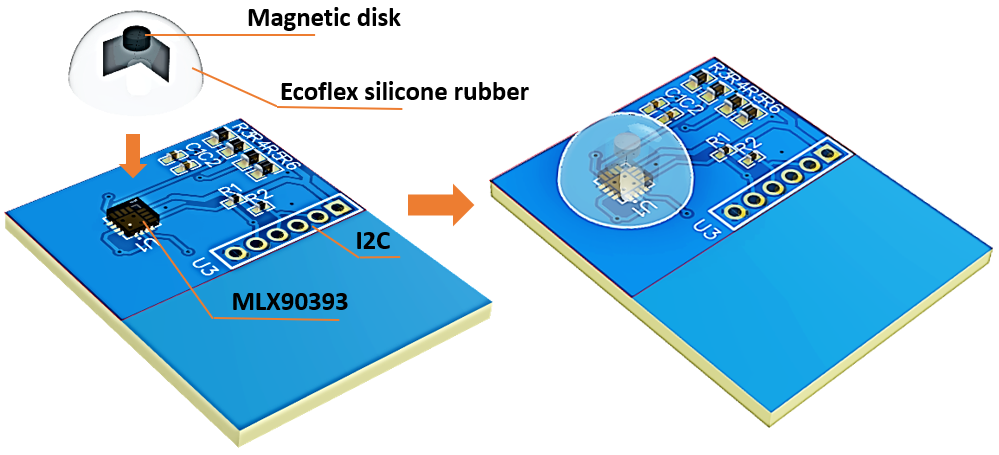}
    \caption{Electronic design of touch sensor.}
    \label{fig:sensorchart}
\end{figure}

\begin{figure*}[h!]
% \vspace{-2ex}
    \centering
    \includegraphics[width=0.99\textwidth]{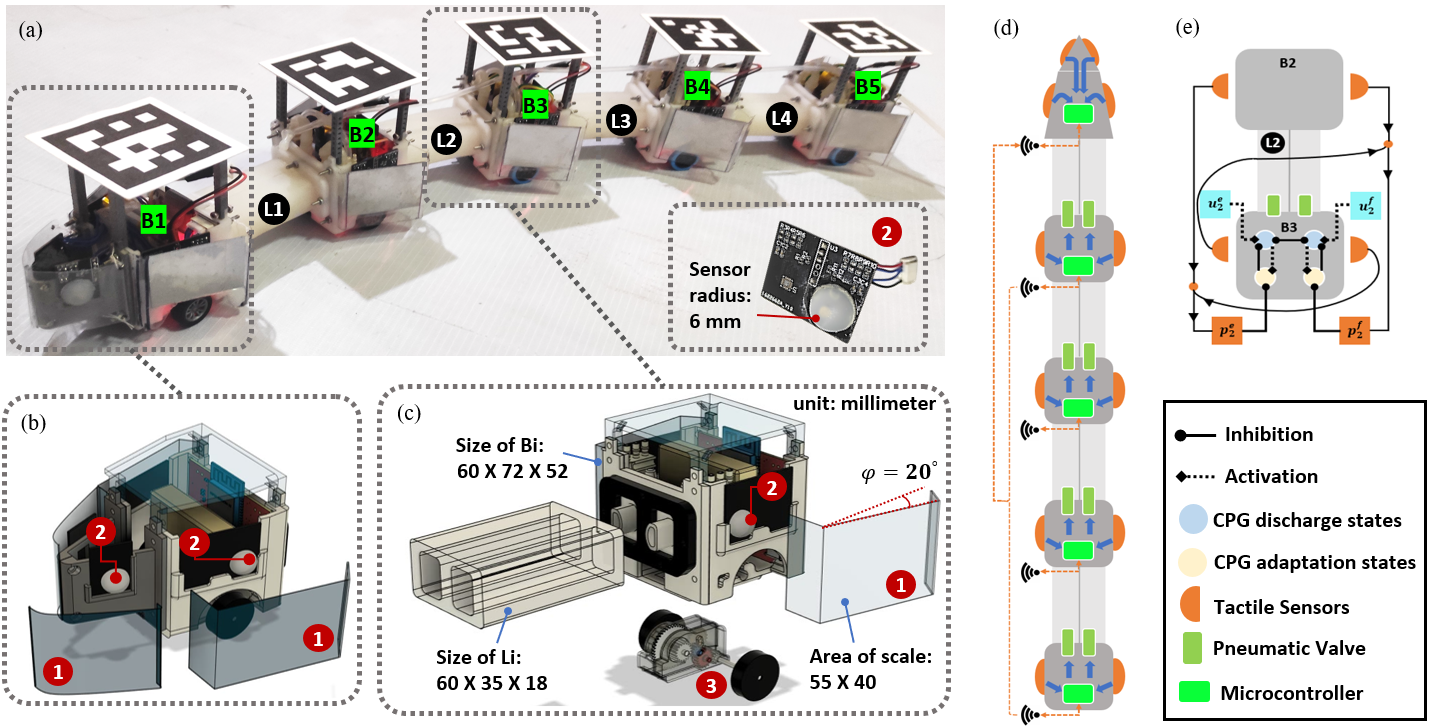}
    \caption{(a) Soft robotic snake (soft snake robot) in reality. (b) The 3D model of rigid head (left) and (c) rigid body (right). (d) Signal communication flow of soft snake robot circuit. (e) Example of sensor-CPG connection model for one link of an soft snake robot.}
    \label{fig:system_view}
    % \vspace{-3ex}
\end{figure*}

In contact-aware robot locomotion, the tactile sensors are expected to detect the contact force timely. Other desired properties can be low-cost, small-sized, durable, accurate, deformable, and customizable. With these requirements in mind, we choose a magnetic field soft tactile sensor based on \cite{wang2016touchsensor}. As shown in Fig.~\ref{fig:sensorchart}, the major component of the soft tactile sensor is comprised of a small magnet cylinder disk (with $2$ mm diameter and a height of $1$ mm) and a Melexis MLX90393 Hall effect module ($3 \text{mm} \times 3\text{mm} \times 0.8\text{mm}$, QFN-16
package) separated by a hemisphere-shaped elastomer (made of Ecoflex\texttrademark~00-30 silicone rubber). The magnet piece is sealed in the elastomer through molding of the silicone first and then the elastomer is glued to the top of the hall sensor on the printed circuit board (PCB). The detailed fabrication steps are similar to \cite{wang2016touchsensor}. The working principle of this tactile sensor is based on the detection of the presence and magnitude of a magnetic field using the Hall effect. The magnetic field varies when the elastomer deforms and causes positional changes in the small magnet disk inside the elastomer. These changes can be detected and calculated by the hall sensor. The data collected by the hall sensor is sent to the motherboard via Inter-Integrated Circuit (I2C) bus. 

According to \cite[(12),(13),(14)]{wang2016touchsensor}, the three direction forces of the tactile sensor are calculated by 
\begin{align*}
    \begin{split}
    &F_z = \Sigma_{k=0}^{n} \Sigma_{i=0}^{k} C_{zj} B_{z}^i B_r^{(k-i)}, \quad j = 1,\ldots \frac{(n+1)n}{2}\\
    &F_r = \Sigma_{k=0}^{n} \Sigma_{i=0}^{k} C_{rj} B_{z}^i B_r^{(k-i)}, \quad j = 1 ,\ldots \frac{(n+1)n}{2}\\
    &F_x = \frac{B_x}{\sqrt{B_x^2+B_y^2}} F_r,\\
    &F_y = \frac{B_y}{\sqrt{B_x^2+B_y^2}} F_r.
    \end{split}
\end{align*}
Where $F_z$ is normal magnetic force, $F_r$ is shear magnetic force which can be decomposed to $F_x$ and $F_y$. Parameters $B_z$ and $B_r$ are the normal and shear magnetic intensity. $B_r$ can be decomposed to $B_x$ and $B_y$. $C_{zj}$ and $C_{rj}$ are the $j$-th coefficients of best fitting polynomials of  $F_z$ and  $F_r$ calculated by moving least squares (MLS) method, and $n$ is the order of the polynomial in the MLS method. 

For the proposed control design, the accuracy of the force direction is not a strict requirement due to the scale structure. We thus focus on the detection of contact events and simplify measuring the magnitude of the force using 
\begin{align}
    F = \sqrt{F_x^2 + F_y^2 + F_z^2}.
\end{align}
Furthermore, we introduce a sigmoid function to normalize  the sensory value of the soft tactile sensor, such that
\begin{align}
\label{eq:sigmaf}
    \sigma(F) = \frac{1}{1-\exp^{-a|F|}},
\end{align}
where $a\in \reals$ is a positive constant.

% \jf{I don't udnerstand this what is $N$? how is this tune the sensitivity of the sensor?}

\subsection{Deployment of Scale-covered Tactile Sensors}
\label{sec:sensorscale}

% model the anisotropic friction of real snakes

The soft snake robot consists $4$ pneumatically actuated soft links (L1$\sim$L4 in Fig.~\ref{fig:system_view}a)\cite{renato2019, xliu2020, xliu2023}. The links are connected by $5$ rigid bodies (B1$\sim$B4 in Fig.~\ref{fig:system_view}a) enclosing the electronic components that are necessary to control the snake robot. Only one chamber on each link is active (pressurized) at a time. The mechanical design of the soft pneumatic actuators is discussed in \cite{luo2017toward}. In this work,
we install elastic one-direction wheels on the rigid part of the soft snake robot to realize anisotropic friction property (component 3 in Fig.~\ref{fig:system_view}c, the contribution of the elastic one-direction wheels on improving the energy efficiency of contact-aware locomotion of the soft snake robot can be found in a supplementary material\footnote{The supplementary document is available at \url{https://shorturl.ac/7bmmw}}). 
 
 There are in total $12$ tactile sensors installed on the robot. As shown in Fig.~\ref{fig:system_view}a,b,c, the components marked  number 2 are the installation positions of the tactile sensors. On top of each tactile sensor, the components marked number 1 are the scales. Each scale is made of two layers of materials -- an acrylic layer attached by a steel plate layer. The scales are designed for three major purposes:
 \begin{itemize}
    \item To significantly increase the contact sensitivity and effective sensing area of the snake robot (contact area expands about $20$ times, according to Fig.~\ref{fig:system_view}c).
    \item To reduce friction resistance on the contact surface (friction coefficient reduced from 1.7 of dry silicone to around 0.3 of polished acrylic board).
    \item To protect the silicone tactile node from frequent collisions.
 \end{itemize}

\begin{figure}[h!]
    \centering
\includegraphics[width=0.55\columnwidth]{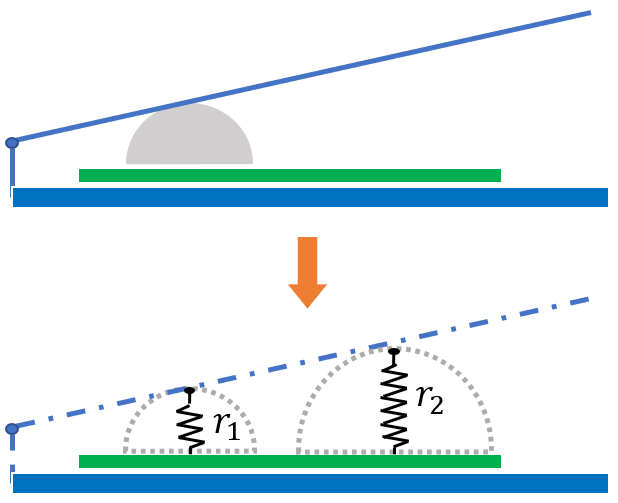}
    \caption{Tactile sensor$+$scale structure (top) versus its approximation in simulation (bottom).}
    \label{fig:sensorsim}
    \vspace{-2ex}
\end{figure}

In order to simulate the robot for reinforcement learning and sim-to-real transfer of the learned controller, we developed a physics-based high-fidelity simulator that models the inflation and deflation of the air chamber and the resulting deformation of the soft bodies with tetrahedral finite elements \cite{renato2019}. To simplify the tactile sensing function of the scale structure in simulation, we use two hemisphere elastic force fields to model the tactile sensor node$+$scale structure in reality (as shown in Fig.~\ref{fig:sensorsim}). The elastic force fields have equilibrium positions (where elastic force equals zero) everywhere on the surface of the hemispheres and have no friction on the hemispheres. The tactile readings are modeled by the elastic forces when an object's distance is smaller than the radius of any simulated tactile node. In the simulation, the reading of the two hemisphere force fields is added together to simulate the contact force signal of one tactile sensor in the real robot.

Inspired by a previous study on obstacle-aided locomotion of rigid snake robots \cite{liljeback2011experimental}, we approximate the total contact force acted on each rigid body $B_i$ based on the inputs sampled from the force sensors. For simplicity, we reduce the element of the sensory force representation by subtracting the two inputs from a pair of diagonal sensors. Therefore, for the $i$-th rigid body (counted from the head as $1$st rigid body), we have
\begin{align}
\label{eq:contact}
    N_{i} = \begin{cases}
        N_{i1}^e+N_{i2}^e - N_{i1}^f- N_{i2}^f, \quad i=1\\
        N_{i}^e - N_{i}^f, \quad i=2,3,4,5.
    \end{cases}
\end{align}
According to \eqref{eq:sigmaf}, let $F_i^e, F_i^f$ represent the magnitude of contact force detected from the left and right sensor respectively on the $i$-th rigid body of the soft snake robot, then $N_{i}^e = \sigma(F_i^e)$, and $N_{i}^f = \sigma(F_i^f)$ (left, right in reference to the heading direction of the soft snake robot). The head joint is a special case since it has two pairs of sensors installed.
The collection of contact forces in the soft snake robot from head to tail   forms a vector 
\[
    \mathbf{N} = [N_1, N_2, N_3, N_4, N_5]^T.
\]  

When the robot is in contact with an obstacle, the contact force $N_i^q=N$ on each tactile scale occurs as shown in Fig.~\ref{fig:system_view}c. However, due to the smoothness of the scale and reduction of $\varphi$ (the angle between the scale and the rigid body) during the contact, the $N_t$ component of $N_i^q$ on the tangent direction of the scale is small. As a result, we assume $|N_t|$ to be always smaller than the maximum torque of the torsion spring, so that $\tau$ and $N_t$ are in balance. Therefore, $N_t$ is neglected in the simulator and we take $N_i^q\approx N_n$ for simplicity.

\color{black}

\section{Modified Matsuoka Oscillator with Sensory Feedback}
\label{sec:cpg}

In order to effectively integrate the tactile information into the contact-aware locomotion framework of the soft snake robot, we study the effect of an additional variable on the Matsuoka oscillator for handling the feedback force signals from the tactile sensors. In this section, we analyze the properties of our method and the conventional approach  \cite{ijspeert2021science, 2023classicfeedback} from a theoretical perspective.

\begin{figure*}[h!]
\centering
\subfloat[]{
\includegraphics[width=0.99\columnwidth]{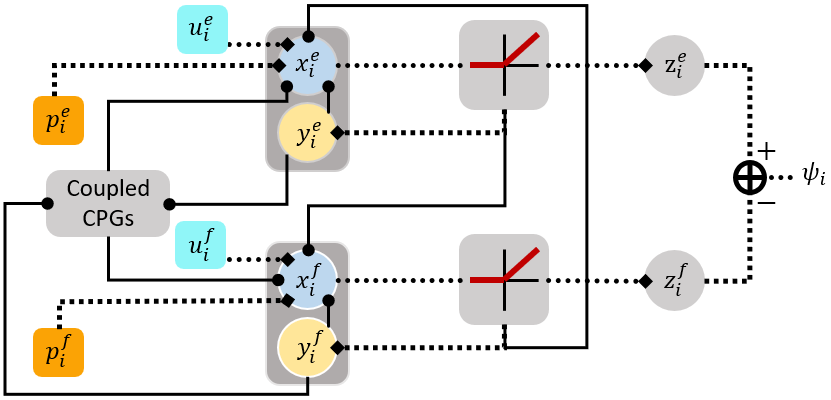}
\label{fig:matsuokaF}}
\hfill
\subfloat[]{
\includegraphics[width=0.99\columnwidth]{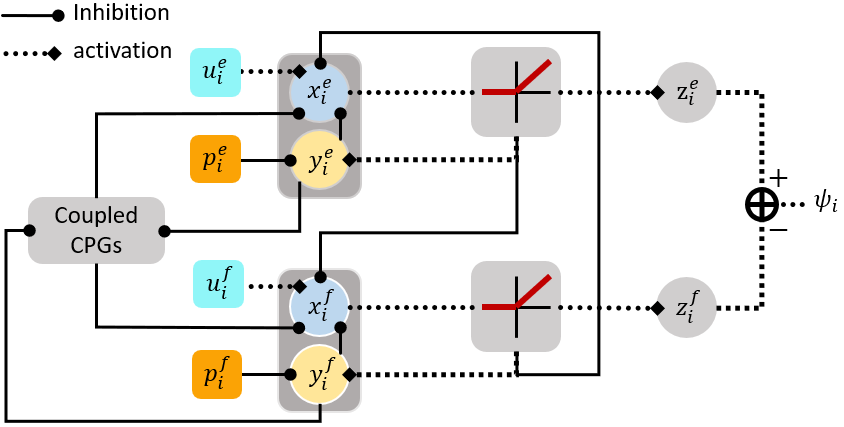}
\label{fig:matsuokaP}}
% \hfill
\caption{Scheme of modified Matsuoka oscillator with different allocation of feedback signals: (a) the conventionally used form (we name as MPF) Matsuoka oscillator and (b) the Adaptation feedback (AF) form Matsuoka oscillator proposed by us in this paper.}
\label{fig:modifiedmatsuoka}
\end{figure*}

In our previous work \cite{xliu2020,xliu2023}, we presented a control scheme that employs sensor-free Matsuoka oscillators to generate undulating control signals as actuation inputs for the soft snake robot to perform Serpentine locomotion. The original Matsuoka oscillator is a piece-wise linear dynamical system, which has the following form:
\begin{align}
\label{eq:matsuoka}
	\begin{split}
		&K_f \tau_r \Dot{x}_i^e = -x_i^e - a z_i^f - b y_i^e - \sum_{j=1}^N w_{ji}y_j^e + u_i^e + c,\\ 
		&K_f \tau_a \Dot{y}_i^e = z_i^e - y_i^e,\\ 
		&K_f \tau_r \Dot{x}_i^f = -x_i^f - a z_i^e - b y_i^f - \sum_{j=1}^N w_{ji}y_j^f + u_i^f + c,\\ 
		&K_f \tau_a \Dot{y}_i^f = z_i^f - y_i^f,
	\end{split}
\end{align} 
where the subscripts $e$ and $f$  represent variables related to the extensor neuron and flexor neuron, respectively. The tuple $(x_i^q, y_i^q)$, $q \in  \{e,f\}$ represents the activation state (or membrane potential) and self-inhibitory state (or adaptation state\cite{matsuoka1985sustained,matsuoka2011analysis}) of $i$-th neuron respectively,  $z_i^q =  \max(0, x_i^q)  $ is the output of $i$-th neuron. Tonic inputs $u_i^e, u_i^f$ are the major coefficients that can be controlled to affect the output bias and amplitude of the Matsuoka oscillator. The frequency ratio $K_f \in \reals$ can be manipulated to affect the natural oscillation frequency of the system. The free-response input is denoted as parameter $c$ in the equation, which is used for amplifying free-response oscillation of the CPG system. The remaining parameters are all constant weights. In system \eqref{eq:matsuoka}, all coupled signals including $x_i^q, y_i^q$ and $z_i^q$ ($q \in  \{e,f\}$) are inhibiting signals (negatively weighted), and only the tonic inputs are activating signals (positively weighted). 

Based on the form of the original Matsuoka oscillator, the question is how to integrate sensory feedback into the Matsuoka CPG system to affect its outputs efficiently? 

A conventional approach is to directly add positive force feedback (as activation signals) to the membrane potential state equations ($\Dot{x}_i^e, \Dot{x}_i^f$) of the original Matsuoka oscillator \cite[(5)]{ijspeert2021science}. Such form of feedback Matsuoka oscillator has been used in some snake robot locomotion studies \cite{ijspeert2021science, 2023classicfeedback}, where the tonic inputs (for locomotion control) of the CPG systems in these applications are mostly constant or regular sinusoidal waves. We summarize the dynamic equations of the conventional feedback Matsuoka oscillator as follows:

\noindent \textbf{Membrane Potential Feedback Form Matsuoka Oscillator: }
\begin{align}
\label{eq:matsuokaF}
	\begin{split}
		&K_f \tau_r \Dot{x}_i^e = -x_i^e - a z_i^f - b y_i^e - \sum_{j=1}^N w_{ji}y_j^e + u_i^e + b p_{i}^e + c,\\ 
		&K_f \tau_a \Dot{y}_i^e = z_i^e - y_i^e,\\ 
		&K_f \tau_r \Dot{x}_i^f = -x_i^f - a z_i^e - b y_i^f - \sum_{j=1}^N w_{ji}y_j^f + u_i^f + b p_{i}^f + c,\\ 
		&K_f \tau_a \Dot{y}_i^f = z_i^f - y_i^f,
	\end{split}
\end{align} 
where the sensory force feedback signals are represented by $p_i^e$ and $p_i^f$. As shown in Fig.~\ref{fig:matsuokaF}, the reason for naming ``membrane potential feedback form'' (MPF) to this type of Matsuoka oscillator is because the sensory feedback signals are directly added to the membrane potential states as activation (positive) signals. 

However, we have a concern about the above conventional form. In the case when the tonic inputs are complicated wave signals and the sensory feedback are irregular signals, these two types of inputs may intervene each other, and therefore fail to effectively present the impact of sensory feedback to the system output.

In \cite{sato2011}, the authors mentioned the addition of the CPG state coupling terms not only to the fast dynamic states' equations (potential membrane $\Dot{x}_i^e,  \Dot{x}_i^f$) but also to the slow dynamic states' equations (adaptation states $\Dot{y}_i^e,  \Dot{y}_i^f$) of the Matsuoka oscillator with opposite signs to improve the dynamic impact of the coupling signals.  Given the inspiration, we consider whether it is possible to add sensory feedback signals which are external impulse signals to the adaptation states of the Matsuoka oscillator? And should the feedback signals be activating or inhibiting in the adaptation state? Could this modification resolve our previous concerns? Why people didn't try this direction in their contact-aware locomotion studies? After in-depth theoretical analysis and experimental comparison, we construct a novel branch of feedback mechanism in the Matsuoka oscillator as follows.

% Inspired by \cite{sato2011}, we design another novel branch of feedback mechanism in the Matsuoka oscillator as follows.

\noindent \textbf{Adaptation Feedback Form Matsuoka Oscillator: }
\begin{align}
\label{eq:matsuokaP}
	\begin{split}
		&K_f \tau_r \Dot{x}_i^e = -x_i^e - a z_i^f - b y_i^e - \sum_{j=1}^N w_{ji}y_j^e + u_i^e +c,\\ 
		&K_f \tau_a \Dot{y}_i^e = z_i^e - y_i^e- p_i^e,\\ 
		&K_f \tau_r \Dot{x}_i^f = -x_i^f - a z_i^e - b y_i^f - \sum_{j=1}^N w_{ji}y_j^f + u_i^f + c,\\ 
		&K_f \tau_a \Dot{y}_i^f = z_i^f - y_i^f- p_i^f,
	\end{split}
\end{align} 
In this design, the tonic inputs $u_i^e, u_i^f$ as well as the free oscillation tonic input $c$ are still added to the potential membrane states ($x_i^e,  x_i^f$) as fast dynamic inputs, while the sensory feedback $p_i^e$ and $p_i^f$ are added to the equations of adaptation states ($y_i^e,  y_i^f$) of Matsuoka oscillator as slow dynamic feedback inputs (see Fig.~\ref{fig:matsuokaP}). In this paper, we name this version of the Matsuoka oscillator as the adaptation feedback (AF) form Matsuoka oscillator.

% \begin{figure}[h!]
%     \centering
%     \includegraphics[width=1.0\columnwidth]{figs/pbiaslinearregression.png}
%     \caption{world coordinate.}
%     \label{fig:biasp}
%     \vspace{-2ex}
% \end{figure}

% In our previous work \cite{xliu2023}, the linear relationship between the bias of tonic inputs and the bias of CPG outputs has been proved. The result shows that the tonic inputs $u_i^e, u_i^f$ can be easily maneuvered by the intelligent control commands to steer the serpentine locomotion of the soft snake robot. In this paper, we further show that the sensory feedback signals has similar properties to the tonic inputs signals. And with stronger impact to the CPG output, the feedback signals are able to override the original oscillation signal when contacts happen.

To explore the feasibility of the AF form Matsuoka oscillator, and find out the advantage of the AF form design, we discuss the difference between the AF and MPF form of Matsuoka oscillator when both the tonic inputs and the sensory feedback signals are variables. The discussion is organized by the following derivations:

Considering AF form Matsuoka oscillator described in system \eqref{eq:matsuokaP} and MPF form Matsuoka oscillator method described in system \eqref{eq:matsuokaF}. 

First, we set $x_i = x_i^e - x_i^f$, $y_i = y_i^e - y_i^f$, $z_i = z_i^e - z_i^f$, $u_i = u_i^e - u_i^f$, $p_i = p_i^e - p_i^f$. By taking subtraction between flexor and extensor in \eqref{eq:matsuokaP} and neglecting phase-related coupling terms from other primitive CPGs, we have 
\begin{align}
    \label{eq:mergedmatsuokaP}
    &K_f \tau_r \frac{d}{dt}x_i = -x_i + a z_i - b y_i + u_i\\ \nonumber
    &K_f \tau_a \frac{d}{dt}y_i = z_i - y_i - p_i.
\end{align}

Similarly, \eqref{eq:matsuokaF} can be simplified to
\begin{align}
    \label{eq:mergedmatsuokaF}
    &K_f \tau_r \frac{d}{dt}x_i = -x_i + a z_i - b y_i + u_i + b p_i\\ \nonumber
    &K_f \tau_a \frac{d}{dt}y_i = z_i - y_i.
\end{align}

If $x_i^e$ and $x_i^f$ satisfy the \textit{perfect entrainment assumption} \cite{matsuoka2011analysis}, we have $z_{{\cal_F}_i} = K(r_x) x_{{\cal_F}_i}$, where $r_x$ is the amplitude bias of $x_i$, and $K(\cdot)$ is the amplitude coefficient function of $x_{{\cal_F}_i}$ \cite[(B.4)]{xliu2023}. The subscript ${\cal_F}_{i}$ indicates the fundamental sinusoidal and constant component in Fourier expansion of the corresponding variable. Without loss of generality, let $K_f = 1$, Eq.~\eqref{eq:mergedmatsuokaP} can be further simplified to
\begin{align}
\label{eq:naivematsuokaP}
    &\tau_r \frac{d}{dt}x_{{\cal_F}_i} + x_{{\cal_F}_i} = a K(r_x) x_{{\cal_F}_i} - b y_{{\cal_F}_i} + u_{{\cal_F}_i}\\ \nonumber
    &\tau_a \frac{d}{dt}y_{{\cal_F}_i} + y_{{\cal_F}_i} = K(r_x) x_{{\cal_F}_i} - p_i.
\end{align}
And \eqref{eq:mergedmatsuokaF} can be further simplified to
\begin{align}
    \label{eq:naivematsuokaF}
    &\tau_r \frac{d}{dt}x_{{\cal_F}_i} + x_{{\cal_F}_i} = a K(r_x) x_{{\cal_F}_i} - b y_{{\cal_F}_i} + u_{{\cal_F}_i} + b p_i\\ \nonumber
    &\tau_a \frac{d}{dt}y_{{\cal_F}_i} + y_{{\cal_F}_i} = K(r_x) x_{{\cal_F}_i}.
\end{align}
Next, an ordinary differential equation can be obtained by merging the two equations in \eqref{eq:naivematsuokaP} as, 
\begin{align}
\label{eq:odematsuokaP}
    &\tau_r \tau_a \frac{d^2}{dt^2} x_{{\cal_F}_i} + (\tau_r + \tau_a - \tau_a a K(r_x)) \frac{d}{dt} x_{{\cal_F}_i} \\ \nonumber
    &+ ((b-a) K(r_x) + 1)x_{{\cal_F}_i} = \tau_a\frac{d}{dt}u_{{\cal_F}_i} + u_{{\cal_F}_i} + b p_i.
\end{align}
Merging the two equations in \eqref{eq:naivematsuokaF} yields
\begin{align}
\label{eq:odematsuokaF}
    &\tau_r \tau_a \frac{d^2}{dt^2} x_{{\cal_F}_i} + (\tau_r + \tau_a - \tau_a a K(r_x)) \frac{d}{dt} x_{{\cal_F}_i} \\ \nonumber
    &+ ((b-a) K(r_x) + 1)x_{{\cal_F}_i} = \tau_a\frac{d}{dt}u_{{\cal_F}_i} + u_{{\cal_F}_i} + b p_i + \tau_a b \Dot{p_i}.
\end{align}

From the right-hand side of \eqref{eq:odematsuokaF} and \eqref{eq:odematsuokaP}, the derivation \eqref{eq:odematsuokaF} of MPF form Matsuoka oscillator has an additional free term $\tau_a b \Dot{p_i}$ comparing to the derivation \eqref{eq:odematsuokaP} of the AF form Matsuoka oscillator. According to the superpositivity of solutions of the second order ordinary differential equation (ODE), when $p_i$ is a variable with complex waveform (e.g. collision force signals), the interference of $\tau_a b \Dot{p_i}$ will be relatively large.  Concluding the above discussion yields the following remark.

\begin{remark}
\label{re:AFvsMPFvariable}
    For the two types of feedback Matsuoka system (AF form and MPF form) satisfying perfect entrainment condition\cite{matsuoka2011analysis}, when the feedback inputs $p_i^e, p_i^f$ are variables,  the MPF form has an additional input disturbance caused by $\Dot{p_i}$, which could bring overshoot and delay to the system. Thus, the feedback inputs of AF form Matsuoka oscillator are more effective than the feedback inputs of MPF form Matsuoka oscillator.
\end{remark}

\begin{figure}[h!]
    \centering
    \includegraphics[width=1.0\columnwidth]{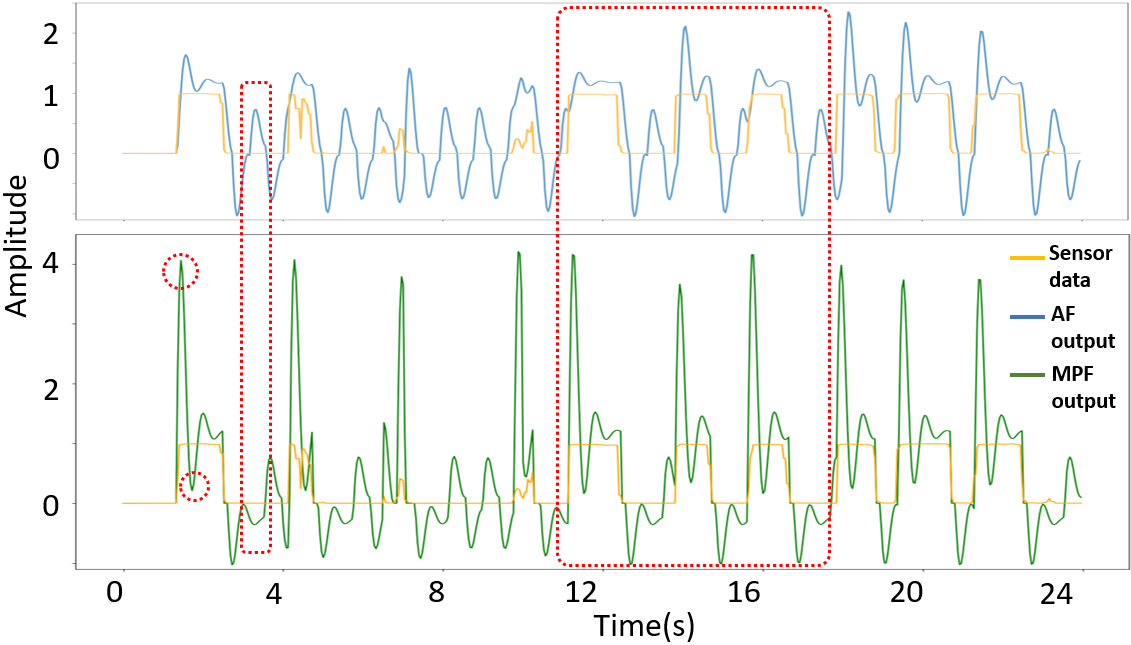}
    \caption{Output of AF form and MPF form Matsuoka oscillator given sensory feedback data.}
    \label{fig:af_mpf_cp}
    % \vspace{-2ex}
\end{figure}

The result of a initial comparison test in Fig.~\ref{fig:af_mpf_cp} further shows the disturbance caused by $\tau_a b \Dot{p}_i$ in MPF form Matsuoka oscillator. In this test, we input the same section of the sensor signal (in orange) to both a primitive AF and a primitive MPF Matsuoka oscillator. The tonic inputs are kept constant for both forms of CPGs. From the MPF output curve, much larger overshoots compared to the AF form output are observed every time a significant contact signal is detected. In every recovery phase after the sensory input vanishes, the output signal of MPF form Matsuoka oscillator also get delayed before recovering to the normal oscillation, while this problem is not observed in the output of AF form Matsuoka oscillator. These observations verify the conclusion in Remark~\ref{re:AFvsMPFvariable}.

It is also worth noted that, in the AF form Matsuoka oscillator, the sensory feedback signals should be inhibiting instead of activating. The detailed illustration of this design can be found in Appendix~\ref{app:remark}.

Next, in the AF form Matsuoka oscillator, in order to compare the impact of tonic inputs $u_i^e, u_i^f$ and sensory feedback inputs $p_i^e, p_i^f$ to the output amplitude bias, we introduce the following proposition,

 \begin{figure*}[h!]
    \centering
    \includegraphics[width=0.95\textwidth]{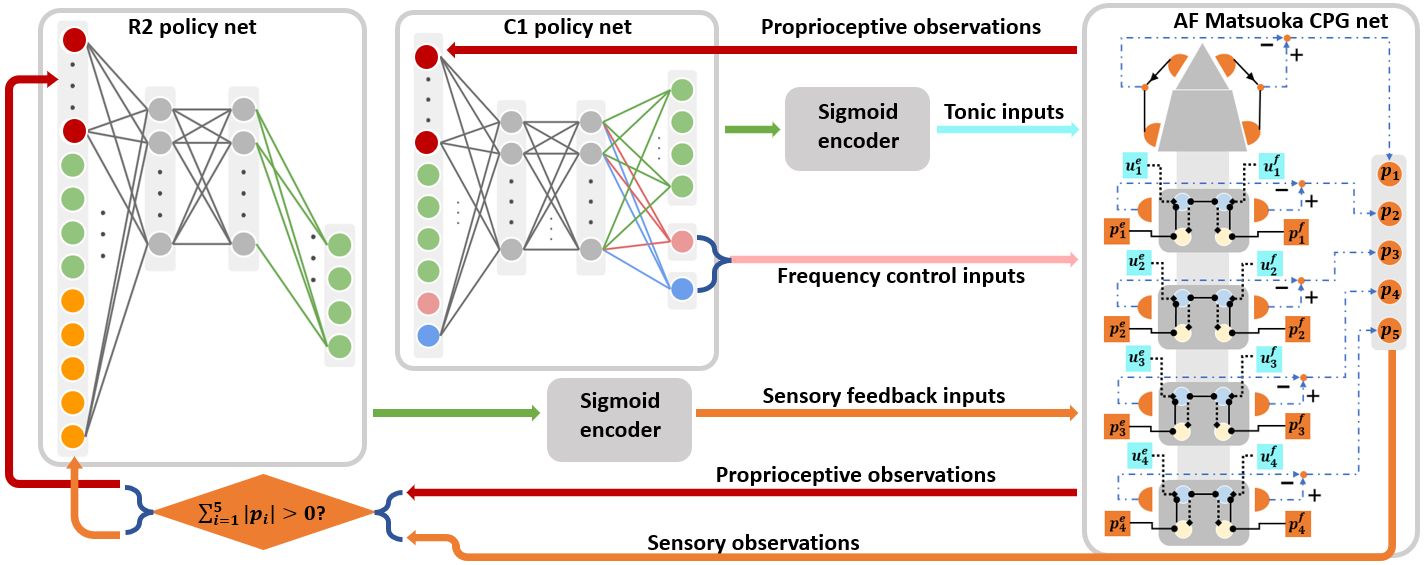}
    \caption{AF-learning control scheme.}
    \label{fig:aflearningscheme}
\end{figure*}

 \begin{figure*}[ht]
    \centering
    \includegraphics[width=0.99\textwidth]{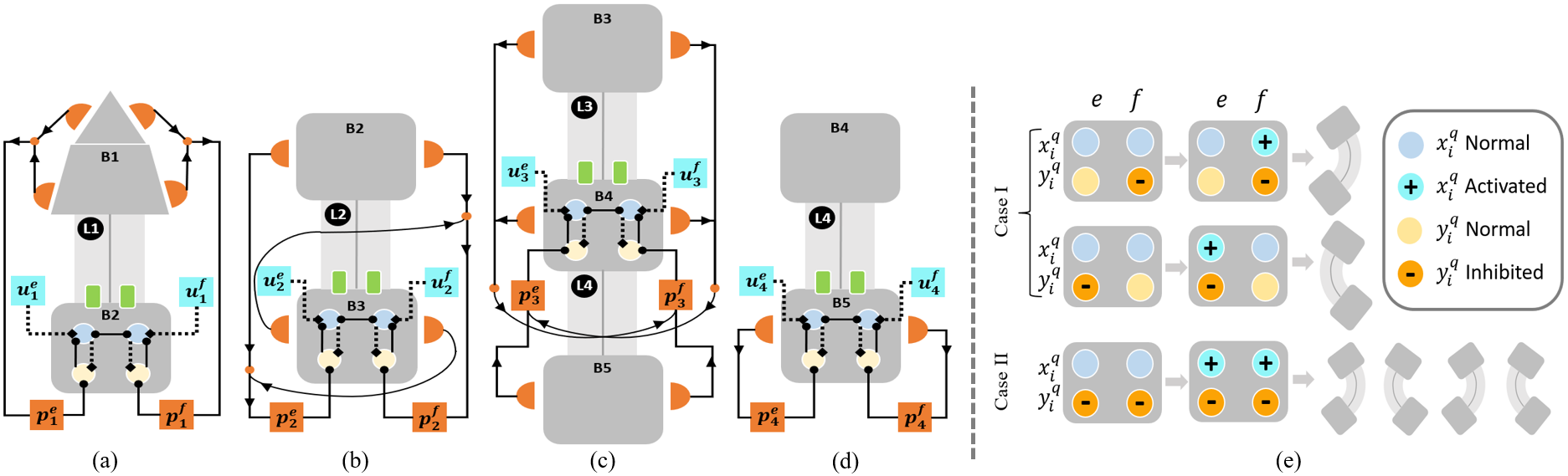}
    \caption{(a)-(d) Local reflexive structure of modified Matsuoka CPG network, and (e) Soft body actuation by the modified Matsuoka oscillator under contact.}
    \label{fig:modifiedmatsuokareflexivemethod}
\end{figure*}

\begin{proposition}
\label{prop:matsuokaduty}
    If an AF form Matsuoka oscillator satisfies the following conditions:
    \begin{inparaenum}[1)]
        \item The dynamical model of the primitive Matsuoka oscillator is harmonic, 
        \item the tonic inputs $u_i^e$ and $u_i^f$ are square wave signals and are complementary to each other,
        \item the sensory feedback signals $p_i^e$ and $p_i^f$ are square wave signals and are complementary to each other.
        \item $u_i^e$ is entrained with $z_i^e$, and $u_i^f$ is entrained with $z_i^f$,
    \end{inparaenum}
    then the oscillation bias of $z_i$ and the bias of $u_i$ satisfies the following relationship, 
    \begin{align}
    \label{eq:dutybiasP}
    \mbox{bias}(z_i) = \frac{1+2 m}{b-a+2}\mbox{bias}(u_i) + \frac{b}{b-a+2} \mbox{bias}(p_i),
    \end{align}
    where  $z_i=z_i^e-z_i^f$, $u_i = u_i^e-u_i^f$, $p=p_i^e-p_i^f$, and 
    \[
        m = \frac{1}{\pi}\frac{1}{2K_n-1+\frac{2}{\pi}(a+b)\sin^{-1}(K_n)}
    \]
     is a constant coefficient ($r_i$ indicates amplitude of state $x_i$). 
\end{proposition}

\begin{proof}
    (See Appendix~\ref{app:theory}.)
\end{proof}

Proposition~\ref{prop:matsuokaduty} shows that in the AF form Matsuoka oscillator, there exists a binary linear relationship between the bias of $u_i$, $p_i$, and the bias of $z_i$. Because the range of $u_i$ and $p_i$ are both limited within $[0, 1]$,  the impact of $p_i$ is larger than $u_i$ when the coefficient of $\mbox{bias}(p_i)$ is larger than the coefficient of $\mbox{bias}(u_i)$. In this paper, the constant parameters of the Matsuoka oscillator are configured according to Table~\ref{tab:config}. Under this condition, we have $b>>1+2m$. The above discussion indicates that when contacts occur, $\mbox{bias}(p_i)$ makes major impact to $\mbox{bias}(z_i)$.

% \begin{figure*}[h!]
% \centering
% \subfloat[]{
% \includegraphics[width=0.28\columnwidth]{figs/reactive_methods_timelag.png}
% \label{fig:escp1}}
% \hfill
% \subfloat[]{
% \includegraphics[width=0.28\columnwidth]{figs/reactive_methods_amplitude_static.png}
% \label{fig:escp2}}
% \hfill
% \subfloat[]{
% \includegraphics[width=0.82\columnwidth]{figs/reactive_methods_amplitude.png}
% \label{fig:escp3}}
% \hfill
% \subfloat[]{
% \includegraphics[width=0.4\columnwidth]{figs/reactive_methods_success_rate.png}
% \label{fig:escp4}}

% \caption{A sample way-point trajectory followed by locally responsive PPOC-CPG controller in a goal-reaching task with a line of obstacles blocking the path. (a) to (d) visualize the distribution of reactive signals along the trajectory to the CPG-controlled actuators from head to tail respectively.}
% \label{fig:escreal}
% \end{figure*}

% \begin{figure}[h!]
% \centering
% \subfloat[]{
% \includegraphics[width=0.9\columnwidth]{figs/compare_p_mode.png}
% \label{fig:escp13}}
% \hfill
% \subfloat[]{
% \includegraphics[width=0.9\columnwidth]{figs/compare_u_mode.png}
% \label{fig:escp23}}
% \hfill
% \subfloat[]{
% \includegraphics[width=0.9\columnwidth]{figs/compare_e_mode.png}
% \label{fig:escp33}}

% \caption{A sample way-point trajectory followed by locally responsive PPOC-CPG controller in a goal-reaching task with a line of obstacles blocking the path. (a) to (d) visualize the distribution of reactive signals along the trajectory to the CPG-controlled actuators from head to tail respectively.}
% \label{fig:mode_signal_comparison}
% \end{figure}
Overall, the properties of AF form Matsuoka oscillator show its flexible and accurate capability of reacting to contact events. Based on this, we can further develop contact-aware controllers for the soft snake robot locomotion.

\section{Design of Controllers}
\label{sec:main}

In the literature on obstacle-aided snake robot locomotion control, there are two ways to process contact sensory feedback signals. One, referred to as \emph{hybrid control} \cite{liljeback2010hybrid}, is the event-triggered hybrid control that utilizes an individual event-triggered controller to optimize the control command only when the contacts happen together with the main locomotion controller that controls the robot at every time step. One special feature of the contact event-triggered controller is that even when only a single part of the robot body is in contact, this controller will send a control command to the whole system. Another method referred to as \emph{reflex control}\cite{kano2012local}, is to use the local reflexive method to set distributed rules for the snake links such that only a few neighboring links will react to the sensory feedback, and such reactions are independent to the main locomotion controller.  

Taking inspiration from the above two directions and combining them with AF form Matsuoka CPG system respectively, we propose two different contact reactive control schemes: 1) the AF-learning method that adopts the concept of \emph{hybrid control} by utilizing a learning-based contact event-triggered controller to observe the sensory data and operate the sensory feedback coefficients in the Matsuoka CPGs of the soft snake robot (as shown in Fig.~\ref{fig:learningreflex}); and 2) the AF-local method that adopts the concept of \emph{reflex control} by directly connecting contact sensory signals to the sensory feedback coefficients of the Matsuoka CPGs following local reflexive rules (as shown in Fig.~\ref{fig:localreflex}). Each of the two methods has its own specialties: although the AF-learning method is computationally expensive, it can achieve greater performance through training. On the other hand, although the local reflexive method has a fixed heuristic rule, it is more light-weighted and robust to the disturbances because of its local reactive property. 

As a result, the investigation and comparison of the above two different contact reactive control designs are necessary to comprehensively verify the advantages of the sensory feedback mechanism in the AF form Matsuoka oscillator for the soft snake robot's contact-aware locomotion.

% \subsection{Integration of Feedback Matsuoka Oscillator in the Learning-based soft snake robot Locomotion Control}

% However, their performance will be compared with the proposed AF series method in the experiment section.

% It is noted that when the AF form Matsuoka CPG system is replaced by MPF form Matsuoka CPG system, the above two methods turn into MPF-local method and MPF-learning method, which are not the focus of this paper. However, their performance will be compared with the proposed AF series method in the experiment section.

\subsection{Event-triggered learning-based sensory reactive controller with AF Form CPG System}
\label{sec:rl}

In the narrative of this paper, the extensor and flexor in the CPG system are assigned with left and right sides of the snake robot respectively (taking the heading direction of the robot for reference). 

In the AF-learning method, we introduce the concept of hybrid control to a model-free learning-based control framework, which is composed of two controllers in the contact-aware goal-tracking task of soft snake robot -- including an RL controller for goal-tracking locomotion named C1, and an event-triggered RL controller for contact reactive control named R2, which only outputs actuator signals when the contact event-triggering condition is satisfied. The scheme of the controller is shown in Fig.~\ref{fig:aflearningscheme}. 

We use a goal-tracking controller developed in our previous work\cite{xliu2023}, called Free-response Oscillation Constrained Proximal Policy Optimization Option-Critics with Central Pattern Generator (FOC-PPOC-CPG), as the C1 controller. The C1 controller takes fully proprioceptive observations of the soft snake robot's dynamic states and outputs control commands by manipulating the tonic inputs of the CPG network as low-level primitive actions and the frequency ratio of the CPG network as high-level options.

For R2 controller, we define the \emph{contact event-triggering condition} as follows: 
At each time step, given the contact force vector $\mathbf{f}$ and contact detection threshold $\epsilon$. The event-triggering condition for the contact-aware scenario is $\norm{\mathbf{f}} > \epsilon$. When the event-triggering condition is satisfied, R2 is triggered to join the manipulation of the CPG system. 

Although it is not necessary for R2 to use the same learning algorithm as C1, for simplicity we also train R2 with the PPOC-CPG framework, which shares the same reward function and AF form CPG system with C1, but with different observation states and actions. 

In the obstacle-based locomotion scenario, there are in total $19$ observation states for R2, denoted as $\bm{\zeta} = \{\zeta_1, \zeta_2, ..., \zeta_{19}\}$, where $\zeta_{1:4}$ represents the dynamic state of the robot referenced on the goal position, $\zeta_{5:8}$ represents the real-time body curvature of the $4$ soft links, $\zeta_{9:14}$ contains actions in the last time step including the previous option and the terminating probability, $\zeta_{15:19}$ contains the pre-processed contact forces. Similar to C1, the actions of R2 are mapped to fit the sensory feedback signals of the Matsuoka CPG network of the soft snake robot. Next, we define a four-dimensional action vector $\vec{a}= [a_1, a_2, a_3, a_4]^T \in \reals^4$ and map $\vec{a}$ to sensory feedback vector $\vec{p}$ as follows, 
\begin{equation}
	\label{eq:decoder}
	p_i^e = \frac{1}{1+e^{-a_i}},\text{ and } 
	p_i^f =1-p_i^e, \text{ for } i=1,\ldots, 4.
\end{equation}
This mapping bounds $p_i^e$ and $p_i^f$ within $[0, 1]$.  The sensory feedback input vector $\mathbf{p}$ for the 4-link snake robot is an eight-dimension vector, 
\begin{equation*}
    \mathbf{p} = [p_1^e, p_1^f, p_2^e, p_2^f, p_3^e, p_3^f, p_4^e, p_4^f]^T.
\end{equation*}
% We use the notation $\bm{p^e}$ or $\bm{p^f}$ to represent the vector of all extensor or flexor elements of $\mathbf{p}$ respectively, i.e.,
% \[
%     \bm{p^e} = [p_1^e, p_2^e, p_3^e, p_4^e]^T, \quad \bm{p^f} = [p_1^f, p_2^f, p_3^f, p_4^f]^T.
% \]
The learning process of the whole control scheme (as shown in Fig.~\ref{fig:learningreflex}) is: C1 is first trained in an obstacle-free environment in simulation. After C1 is converged, we fix C1 as a regular controller for goal-tracking purposes. C1 policy is always effective regardless of the triggering of the contact events. Then we train R2 in the environment with randomly generated obstacle mazes in simulation until convergence. R2 is effective only when the contact event-triggering condition is satisfied. According to Remark~\ref{re:AFvsMPFvariable} and Proposition~\ref{prop:matsuokaduty}, when the parameters of the AF form Matsuoka CPG system satisfy Table~\ref{tab:config}, when R2 is effective, it will dominate the control of the CPG system (contact-awareness over goal-awareness).

\subsection{Local Reflexive Control of Contact-aware slithering locomotion with AF Form CPG System}
\label{sec:localreflex}

%  \begin{figure*}[ht]
%     \centering
%     \includegraphics[width=0.99\textwidth]{figs/reflexive_method_illustration.png}
%     \caption{(a)-(d) Local reflexive structure of modified Matsuoka CPG network, and (e) Soft body actuation by the modified Matsuoka oscillator under contact.}
%     \label{fig:modifiedmatsuokareflexivemethod}
% \end{figure*}

%  \begin{figure}[h!]
%     \centering
%     \includegraphics[width=0.99\columnwidth]{figs/cpgturningprinciple.png}
%     \caption{Soft body actuation by the modified Matsuoka oscillator under contact.}
%     \label{fig:modifiedmatsuokasensors}
% \end{figure}

% the property of how sensory feed back $p_i$ and tonic input $u_i$ influence the overall CPG output given the parameter in Appendix~\ref{app:data} (Table~\ref{tab:config})

% For example, in the first row of Fig.~\ref{fig:modifiedmatsuokasensors}, the inhibition of $y_i^f$ caused by $p^f$ will lead to a left bending behavior of the soft link.

In this section, we describe a different learning-based controller, denoted as AF-local because the local reflexive mechanism is based on the AF form Matsuoka oscillator. According to \eqref{eq:dutybiasP} and Proposition~\ref{prop:matsuokaduty}, if the parameters of an AF form Matsuoka oscillator satisfy Table~\ref{tab:config}, we have $b >> 1+2m$, so that the sensory input $p_i$ will play a major role to influence the tonic input $u_i$ when contact events are detected by the tactile sensors. 

We validated the property through experiments. As shown in Fig.~\ref{fig:modifiedmatsuokareflexivemethod}e, Case I describes a situation when there comes a $p_i^q$ signal in only one side, the adaptation variable $y_i^q, (q\in{e,f})$ on the same side is inhibited, and result in activation of the corresponding $x_i^q$ (at the same time the opposite $x_i^q$ state is inhibited). At this moment, no matter what value of tonic input $u_i^e, u_i^f$ are given within their range $u_i^q\in [0, 1]$, the soft actuator will always bend towards the opposite direction of the incoming $p_i^q$ signal. Case II shows the case when both $p_i^e$ and $p_i^f$ are inhibiting $y_i^e$ and $y_i^f$ respectively, both $x_i^e$ and $x_i^f$ will be strongly activated, leading to almost a free-response oscillation output regardless the values of $u_i^e, u_i^f$. According to Proposition~\ref{prop:matsuokaduty}, the oscillation bias of the CPG output in the Case II situation depends on $\mbox{bias}(p_i)$, where $p_i = p_i^e - p_i^f$.

We take the inspiration of the local reflexive mechanism from \cite{kano2012local, Kano2017} such that only the links that are close to a contact sensor may react to its contact events. Due to the differences in structure (antagonistic actuators, partially tunable chambers) between our pneumatically actuated snake robot and the real-time tunable spring actuated snake model described in \cite{kano2012local}, we have our specific principles for constructing the reflexive loop between the sensors and the sensory feedback inputs of the CPG network (Fig.~\ref{fig:modifiedmatsuokareflexivemethod}):
\begin{itemize}
 \begin{figure}[ht]
    \centering
    \includegraphics[width=0.99\columnwidth]{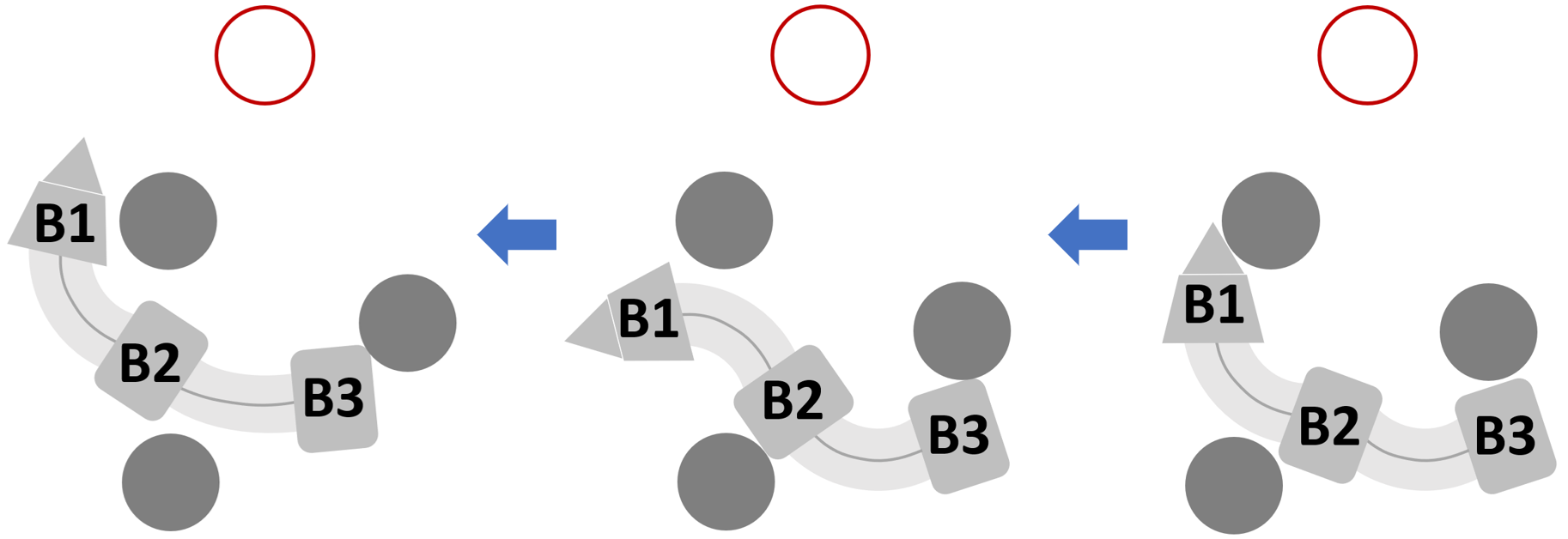}
    \caption{Example of reflexive mechanism on link L1.}
    \label{fig:headjointreflex}
\end{figure}
    \item Since the head of the snake robot maneuvered by the goal-tracking controller is always heading in the goal direction, it could be blocked by an obstacle on the path to the target location if the head cannot properly react to the contact and turn away from the obstacle. Thus the snake robot's head should always bend in the opposite direction to the major contact event, which means that the ipsilateral chamber of L1 link to the contact side of B1's sensor will be actuated. Figure~\ref{fig:headjointreflex} provides an example showing the reflexive behavior of the L1 link when the head sensor on B1 touches an obstacle.
    \item As the rudder of steering and source of propulsion, the snake robot's tail should always push itself against the obstacles to keep oscillation with a larger amplitude. So the ipsilateral chamber of L4 link to the contact side of B5's sensor will be actuated. In addition, in a 4-link soft snake robot, the role of the tail during slithering locomotion should be strengthened to generate more propulsion. Therefore, we extend the impact of the tail sensors to the actuators in the L3 soft link.
     \begin{figure}[ht]
    \centering
    \includegraphics[width=0.65\columnwidth]{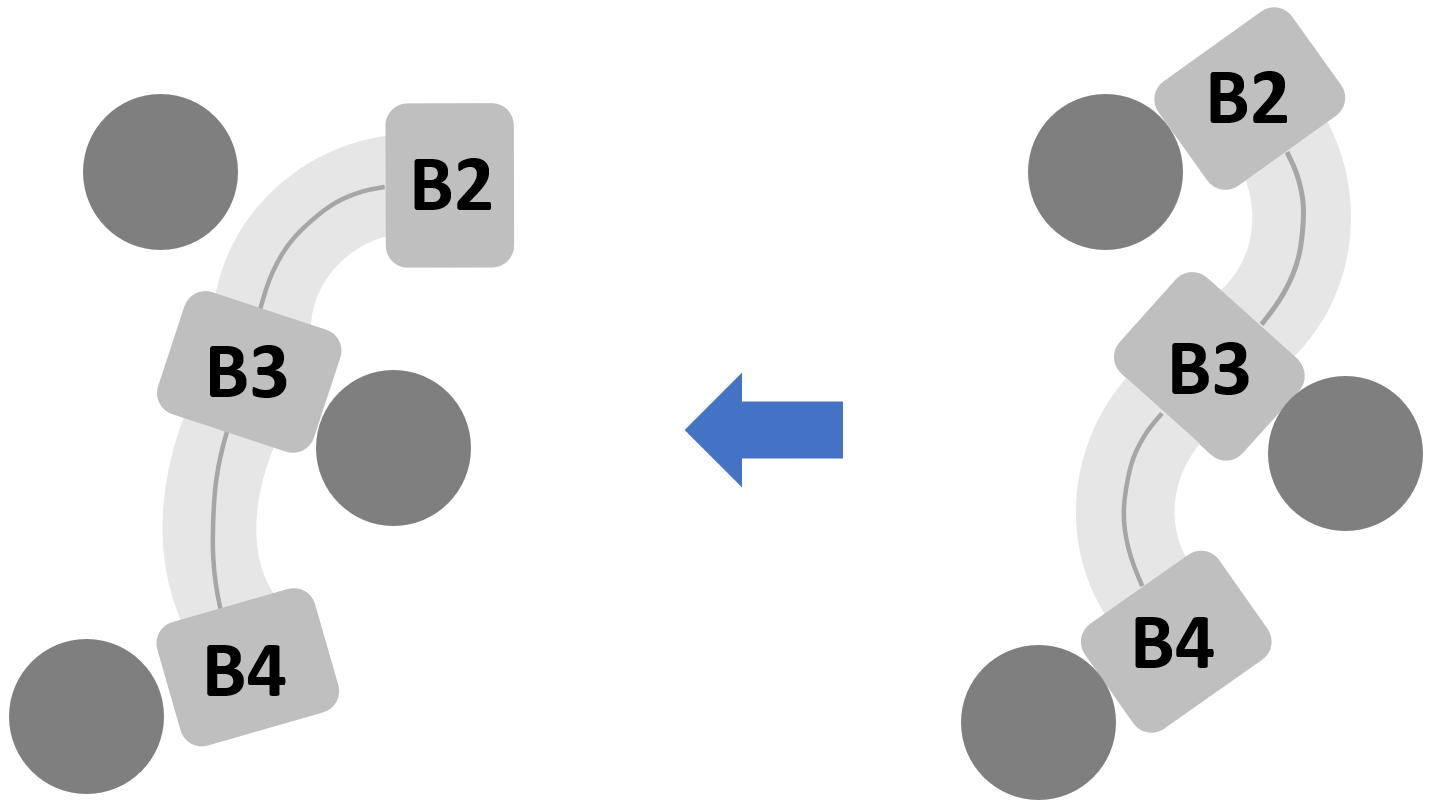}
    \caption{Example of reflexive mechanism on L2 and L3 links.}
    \label{fig:midjointreflex}
\end{figure}
    \item Beyond head and tail links, the other soft body links' CPG nodes should refer to their neighboring sensors to determine their reflexive behaviors accordingly. To design the connection between CPG nodes of the body links and the corresponding sensors, we refer to the jamming case that mostly occurred in the contact-aware locomotion of the soft snake robot. As shown in Fig.~\ref{fig:midjointreflex}, when B2, B3, and B4 rigid parts are in contact with the obstacles on the opposite sides, the situation leads to a typical jamming scenario for our soft snake robot in contact-aware locomotion. In this situation, the L3 link is supposed to decrease its bending curvature to avoid jamming, while L2 should actuate its ipsilateral chamber (extensor) to create more space for free oscillation controlled by the goal-reaching controller.
\end{itemize}

Based on the above features and former experience in designing local reflexive control rules\cite{kano2012local}, we design the topology of the sensor connection to each CPG node in the soft snake robot's ``vertebrate'' system. As shown in Fig.~\ref{fig:modifiedmatsuokareflexivemethod}, the sensory feedback signals $p_i^e, p_i^f$ are normalized when receiving non-zero inputs from the connected tactile sensors.
Before formulating $p_i^e, p_i^f$, we first define set ${\cal D}_{i}, i = 1,2,3,4$ as the set of sensor signals' numbers connected to the $i$-th Matsuoka CPG node. For example, for $3$rd CPG node in Fig.~\ref{fig:modifiedmatsuokareflexivemethod}c, ${\cal D}_{3} = \{3, 4, 5\}$. In addition, we define the connection marker array $\mathbf{J} = [J_1, J_2, J_3, J_4, J_5] = [-1, -1, 1, 1, -1]$. The value in $\mathbf{J}$ is assigned based on the way of connection between the sensors and the CPG network.
\begin{align}
\label{eq:localreflexiverule}
    \begin{split}
        p_i^e = \frac{\Sigma_{k\in {\cal D}_i} I^e(N_k)|N_k|}{\Sigma_{k\in {\cal D}_i} |N_k| + \delta^+},\\
        p_i^f = \frac{\Sigma_{k\in {\cal D}_i} I^f(N_k)|N_k|}{\Sigma_{k\in {\cal D}_i} |N_k| + \delta^+}.
    \end{split}
\end{align}
where $\delta^+ \in \reals^{+}$ is a small positive number to avoid division by zero, and
\begin{align*}
    \begin{split}
        I^e(N_k) = \max\{0, -\mbox{sgn}(J_k N_k)\},\\
        I^f(N_k) = \max\{0, \mbox{sgn}(J_k N_k)\}.
    \end{split}
\end{align*}

More specifically, the mechanism of \eqref{eq:localreflexiverule} acting on the actuators of the soft snake robot can be explained as follows:
\begin{itemize}
    \item In the L1 CPG node (Fig.~\ref{fig:modifiedmatsuokareflexivemethod}a), the sensors are connected to the same side of sensory feedback inputs $p_i^e, p_i^f$ of L1 CPG. When one side of the B1 sensors is in contact, the actuation of the L1 link follows Case I, which bends toward the opposite direction to the triggered sensors.
    \item In L2 CPG node (Fig.~\ref{fig:modifiedmatsuokareflexivemethod}b), the sensors on B2 are connected to the same side of sensory feedback inputs of L2 CPG, while the sensors on B3 are connected to the opposite side of sensory feedback inputs of L2 CPG. When only the B2 or B3 sensor is triggered, or both B2 and B3 receive contact feedback from the opposite side, L2 will behave in Case I. When B2 and B3 have contacts on the same side, both $y_2^e$ and $y_2^f$ will be inhibited, leading to Case II behavior of L2.     
    \item In L3 CPG node (Fig.~\ref{fig:modifiedmatsuokareflexivemethod}c), the sensors on B3 and B4 are connected to the opposite side of sensory feedback inputs of L3 CPG, while the B5 sensors are connected to the same side of sensory feedback inputs of L3 CPG. Consider a single sensor-triggered case, when only the B3, or B4, or B5 sensor is triggered, L3 will also behave in Case I. For two sensors triggered case: when only (B3 and B4) are triggered on the same side, or (B3 and B5) or (B4 and B5) are triggered on the opposite side, L3 will behave in Case I; when only (B3 and B4) are triggered on the opposite side, or (B3 and B5) or (B4 and B5) are triggered on the identical side, L3 will behave in Case II. For three sensor-triggered cases, when B3 and B4 are triggered on the same side opposite to the contact side of B5, L3 will behave in Case I, otherwise, L3 behave in Case II. 
    \item In the L4 CPG node (Fig.~\ref{fig:modifiedmatsuokareflexivemethod}d), the B5 sensors are connected to the same side of sensory feedback inputs of L4 CPG. When one side of the B5 sensor is in contact, the actuation of the L4 link follows Case I, which bends toward the opposite direction to the triggered sensors. 
\end{itemize}

The overall workflow of AF-local is shown in Fig.~\ref{fig:localreflex}. The local reflexive mechanism introduced in this section works independently and maps the tactile sensor data to the sensory feedback signals of the CPG system. In the meantime, the tonic input signals in the same CPG system are controlled by a C1 controller introduced in Section~\ref{sec:rl}, which only focuses on the goal-tracking control of the soft snake robot.

In order to compare AF form Matsuoka CPG system with the conventional MPF form Matsuoka CPG system, we also develop MPF-local and MPF-learning controllers by replacing the AF form Matsuoka oscillator with MPF form Matsuoka oscillator in the two control methods introduced in Section~\ref{sec:localreflex} and Section~\ref{sec:rl}. In  Section~\ref{sec:experiment}, we will comprehensively compare the performance of the AF-local, AF-learning, MPF-local, and MPF-learning methods.

\subsection{Design of the shared reward function}
\label{sec:reward}

We now present our design for the reward function shared by both locomotion and contact-aware controllers. Our design will ensure that by maximizing the discounted sum of reward, the learned controller can achieve efficient locomotion and accurate set-point tracking. 

To improve learning efficiency, we employ a potential field-based reward function. Artificial potential field (APF) is widely applied in planning problems and potential game theory \cite{khatib1986real, park2001obstacle, dong2012strategies} to accelerate the process of searching for the optimal strategy. 
The potential field can be classified into two categories -- the attracting field for target reaching and the repulsive field for obstacle avoidance. The attracting field function is defined as follows
\[
    U_{att}(\mathbf{p}) = \frac{1}{2} k_{att} ||\mathbf{p} - \mathbf{p}_g||^2,
\]
where $\mathbf{p}$ is the coordinate of the agent and $\mathbf{p}_g$ is the coordinate of the goal. Coefficient $k_{att}$ is a positive constant indicating the strength of the attractive potential field. Since the attracting gravity is always pointing toward the goal coordinate from any position of the map, the value of gravity force should be negative. By taking the negative gradient of $U_{att}$, we have the attracting force function
\[
    \mathbf{F}_{att}(\mathbf{p}) = -\nabla U_{att} = -k_{att} (\mathbf{p}-\mathbf{p}_g).
\]

The reward is designed to encourage the goal-reaching, guided by the artificial potential field. We designed the reward to be composed of two rewards: 
\begin{equation}
	R = \omega_1 R_{goal}  + \omega_2 R_{att},
\end{equation}
where $\omega_i,i=1,2,3$ are constant weights. $R_{goal}$ is the termination reward for reaching a circular accepting area centered at the goal. 
\[
R_{goal}= \cos \theta_g \sum_{k=0}^i{\frac{1}{l_k} \mathbf{1}(\rho_g < l_k)}.
\]
where $\theta_g$ is the deviation angle between the locomotion direction of the snake robot and the direction of the goal, $l_k$ defines the radius of the accepting area in task-level $k$, for $k=0,\ldots, i$. $\rho_g = \norm{ \mathbf{p}-\mathbf{p}_g}$ is the linear distance between the head of the robot and the goal, and $\mathbf{1}(\rho_g < l_k)$ is an indicator function to determine whether the robot's head is within the accepting area of the goal. 
$R_{att}$ is the reward function of the attracting potential field:
\[
R_{att} = \mathbf{v} \cdot \mathbf{F}_{att}(\mathbf{p}), 
\]
where $\mathbf{v}$ is the velocity vector. The dot product between $\mathbf{v}$ and the potential field vector represents the extent of the agent's movement on following the potential flow in the task space.

\section{Experiments}
\label{sec:experiment}

\subsection{Signal Communication and Obstructed Environment Setting}
\begin{figure}[h!]
	\centering
 \includegraphics[width=1\columnwidth]{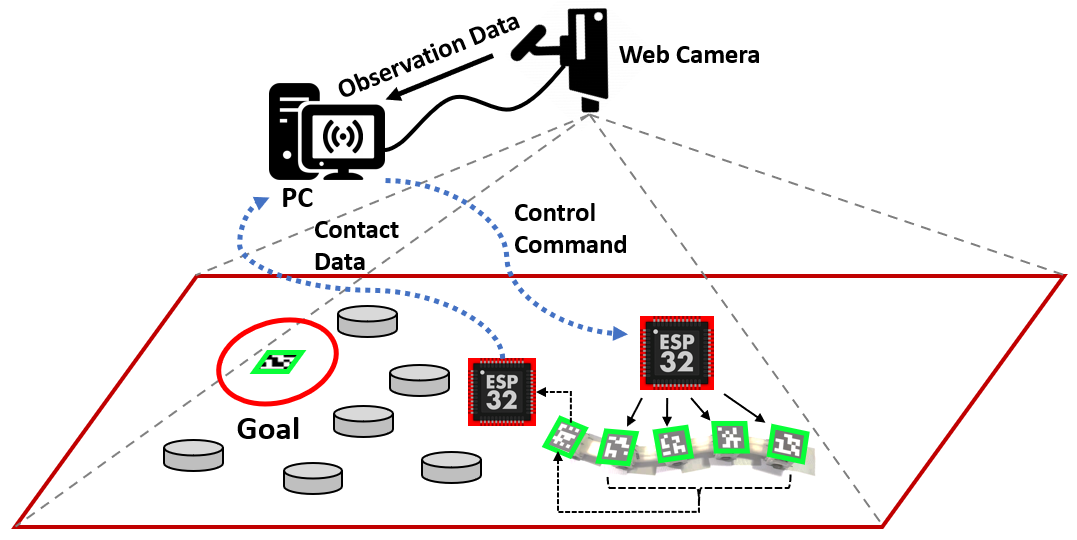}
	\caption{Experiment setup of the contact-aware goal tracking locomotion task in reality.}
	\label{fig:mocap}

\end{figure}

The two-dimensional dynamic states of the soft snake robot are captured and calculated by a web camera (works under $120$ Hz) hanging on the ceiling of the experiment room. We use Aruco \cite{GARRIDOJURADO20142280} to detect and localize QR codes attached to every rigid body of the snake robot and the goal position. Figure~\ref{fig:mocap} shows the experiment setup for the real snake robot goal-reaching tasks. In this work, we update two major parts of the experiment settings:  
\begin{itemize}
    \item In the signal communication part, each ESP32 chip collects contact sensor information from local I2C and shares the data with the head chip through WiFi. In every time step, the head ESP32 chip packs all the sensor data and sends it back to the PC controller. The controller program running on the desktop computer receives the observation states from the web camera and the robot, generates the control commands and passes them to the ESP32 chips on the snake robot through WiFi communication. The ESP32 chips on the snake bodies translate the commands into Pulse Width Modulation (PWM) signals to activate or deactivate the valves \cite{luo2017toward, renato2019} on the snake robot. The communication rate between the PC controller and the snake robot is $30$ Hz.
    \item In the environment setting, a number of tin cans filled with stones and sand are placed in the experiment field as obstacles. Each vertical peg in Fig.~\ref{fig:mocap} represents a cylinder tin can with a diameter of $100$ millimeters (mm) and height of $80$ mm. The average weight of the obstacles is around $1.1$ kg each, and the weight of the soft snake robot is $0.7$ kg (including batteries). It has been tested to ensure that any collision caused by the soft snake robot will not move the obstacles. 
\end{itemize}

\subsection{Simulated Training and Evaluation}
\label{sec:simtrain}

\noindent \textbf{Reinforcement Learning Configuration:} In the simulated training part, the goal-reaching controller C1 is a pre-trained module as configured in \cite{xliu2023}. In this work, the contact-aware regulator R1 in AF-learning and MPF-learning controllers is trained in a goal-reaching task with a randomly generated $6\times 5$ obstacle maze. During the training process of R1, the distance between the robot and the goal is fixed to $1.5$ meters. The deviation angle between the snake robot and the goal is initially sampled from $0\sim60$ degrees with a uniform distribution. In the simulator, the distance between every two obstacles is sampled between $120\sim180$ mm. The coordinate of each obstacle is added by an additional clipped standard Gaussian noise ($\omega \sim \mathcal{N}(0, 1)$, clipped by $-0.01<\omega<0.01$). The method of simulating contact sensors is introduced in Section~\ref{sec:sensorscale}. In order to compensate for the mismatch between the simulation and the real environment, most notably the friction coefficients, we employ a domain randomization technique \cite{tobin2017domain}, in which a subset of physical parameters are sampled from several uniform distributions. The range of distributions of domain randomization (DR) parameters used for training are in Table \ref{tab:dr} (see Appendix~\ref{app:data}). The whole training process of each method runs on $4$ simulated soft snake robots (Rendered by Nvidia Flex) on a workstation equipped with an Intel Core i7-9700K, 32GB of RAM, and one NVIDIA RTX2080 Super GPU. 

\noindent \textbf{Task specification: }In the contact-aware locomotion task, the robot is required to traverse an array of obstacles and reach the randomly generated goals. Similar to the real-world setting in Fig.~\ref{fig:mocap}, there is also an accepting radius in the simulation for each goal-reaching task, which means that the robot needs to be close enough to the goal in order to succeed and receive a terminal reward. At each time step, the robot also receives a reward from the potential field defined in Section~\ref{sec:reward}. If the agent reaches the accepting region of the current goal, a new goal is randomly sampled. In the failing situation, when the robot is jammed by obstacles for a certain amount of time, the desired goal will be re-sampled and updated. The starvation time threshold for failing condition is $900 \text{ ms}$. In addition, if the linear velocity of the snake robot stays negative in the goal direction for over $360$ time steps (each time step is about $20$ \text{ms}), the goal-reaching task is also judged as a failure and trigger the re-sampling of the new task.

\begin{figure}[h!]
	\centering
	\includegraphics[width=0.99\columnwidth]{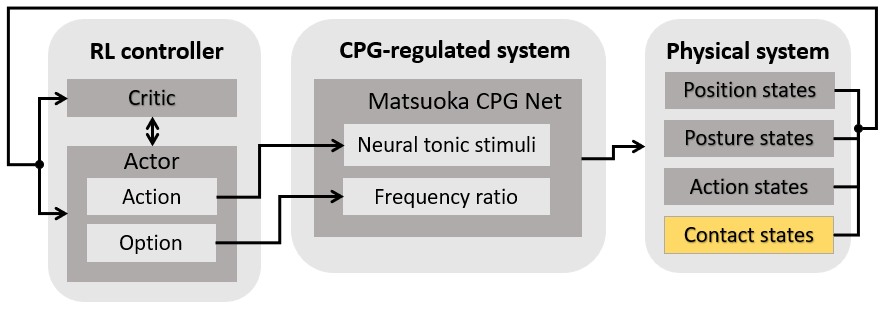}
	\caption{Flow chart of C1$+$ method. Different from C1, C1$+$ has contact information in its observation states, and is further trained in the obstacle-based environment.}
	\label{fig:c1+scheme}
\end{figure}

\begin{figure}[h!]
	\centering
	\includegraphics[width=0.8\columnwidth]{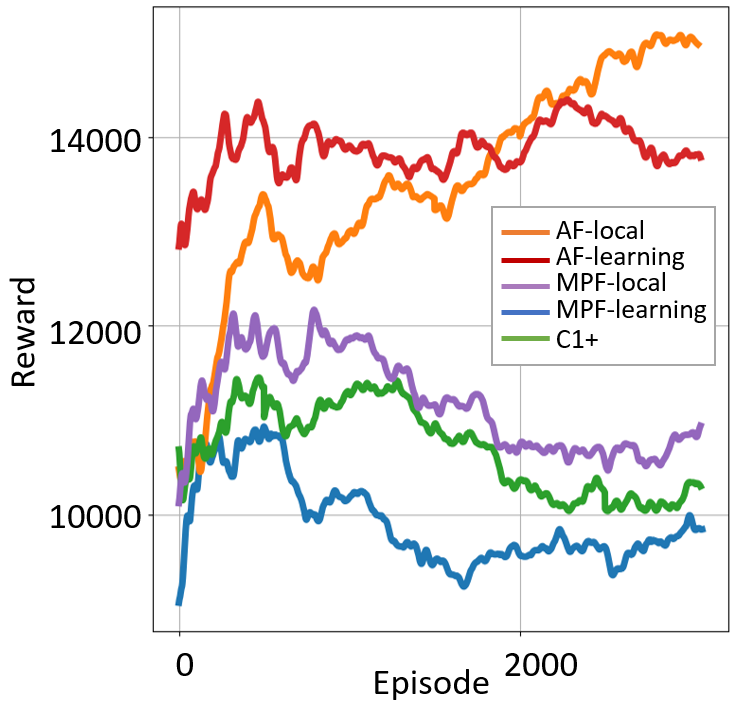}
	\caption{Learning process and evaluation scores comparison of the proposed method recorded in an obstacle-based training environment.}
	\label{fig:compare_result}
\end{figure}

\noindent \textbf{Training Score Comparison:} According to the above task specification, we train the AF-learning, MPF-learning, and C1$+$ methods, and compare their training scores with the evaluation scores of AF-local, MPF-local in the same environment. It is noted that C1$+$ is the contact-aware version of C1 (see Fig.~\ref{fig:c1+scheme}), which directly operates the tonic inputs of the CPG network given the full observation $\zeta_{1:19}$ from the environment. The C1 controller of all AF and MPF methods has been pre-trained in the obstacle-free environment and converged for goal-reaching tasks at the same level. The R2 controllers of AF-learning and MPF-learning methods are then trained in the obstacle-based goal-tracking tasks for $3000$ episodes till convergence, during which the NN parameters of their C1 controllers are fixed. The C1$+$ controller is first trained in the contact-free environment, then transferred to the obstacle-based environment, and is also trained for $3000$ episodes.

From Fig.~\ref{fig:compare_result}, it is observed that the AF-learning method reaches the highest reward and is the only learning method that keeps improving during the learning process. Among the remaining methods, AF-local is the only method with an average reward close to the AF-learning method. This result already shows the advantage of AF related method over the others. It is also noted that although the MPF-local method is evaluated slightly better than the C1$+$ method, the MPF-learning method cannot improve and converge to a higher score than its initial performance and ends up converging to the lowest reward level. According to  Remark~\ref{re:AFvsMPFvariable}, the bad performance of MPF-learning and MPF-local method is possibly due to the influence of $\Dot{p}_i^e, \Dot{p}_i^f$ in the MPF form Matsuoka oscillator, which makes the R2 RL controller more difficult to operate the sensory feedback signals of the CPG system.

\subsection{Performance analysis in real robot experiments}

In this section, we compare the performance of all five methods (mentioned in Section~\ref{sec:simtrain}) in contact-aware soft snake robot locomotion tasks in the real world. Furthermore, we test the performance of the top two methods in more challenging obstacle-based environments.

\subsubsection{Escaping experiment}

\begin{figure}[h!]
    \centering
    \includegraphics[width=0.8\columnwidth]{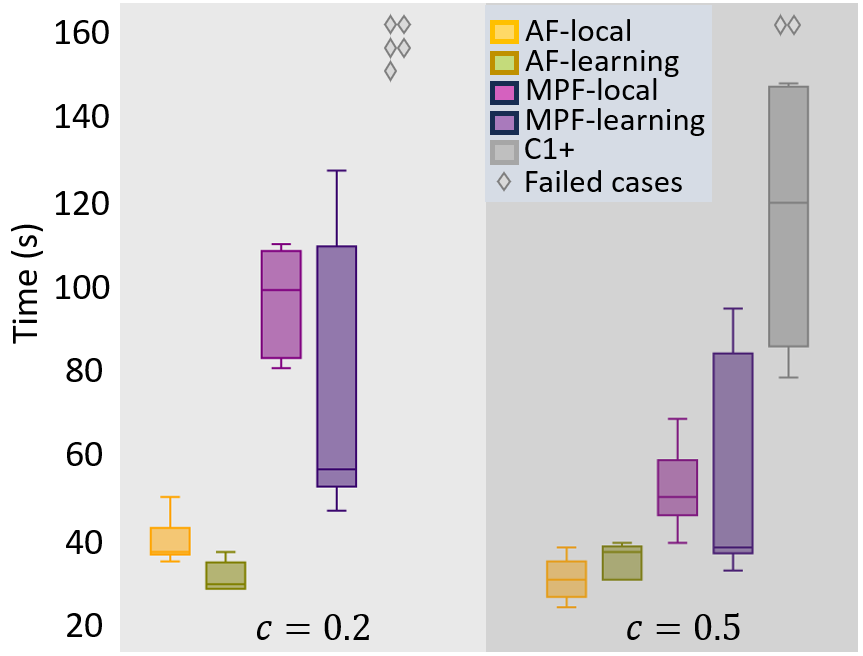}
    \caption{Statistics of escaping time of the proposed methods and the baseline.}
    \label{fig:snakeperformancecompare}
\end{figure} 

\begin{figure*}[h!]
    \centering
    \includegraphics[width=0.8\textwidth]{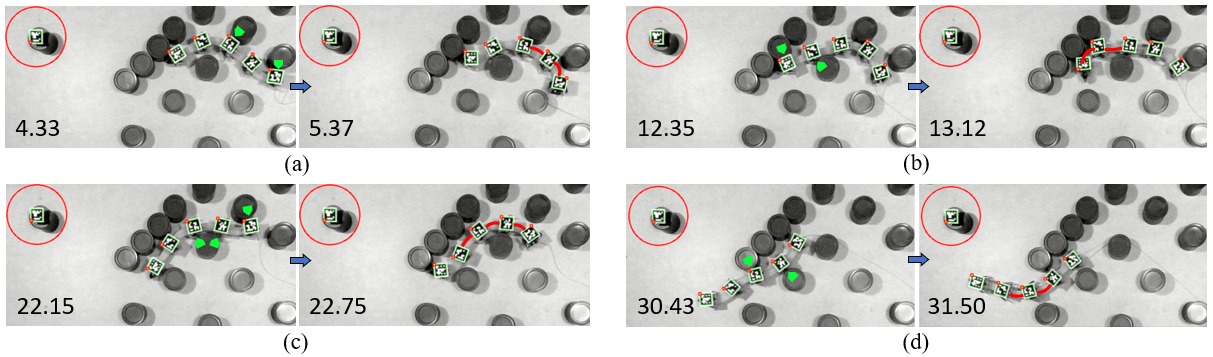}
    \caption{Sample screenshots of performance of the AF-local method in a goal oriented escaping task from the obstacles. Each pair of pictures shows the local reactive behavior of the soft snake robot before and after contacts.}
    \label{fig:videoscreenshot}
\end{figure*}

\begin{figure*}[h!]
    \centering
    \includegraphics[width=0.85\textwidth]{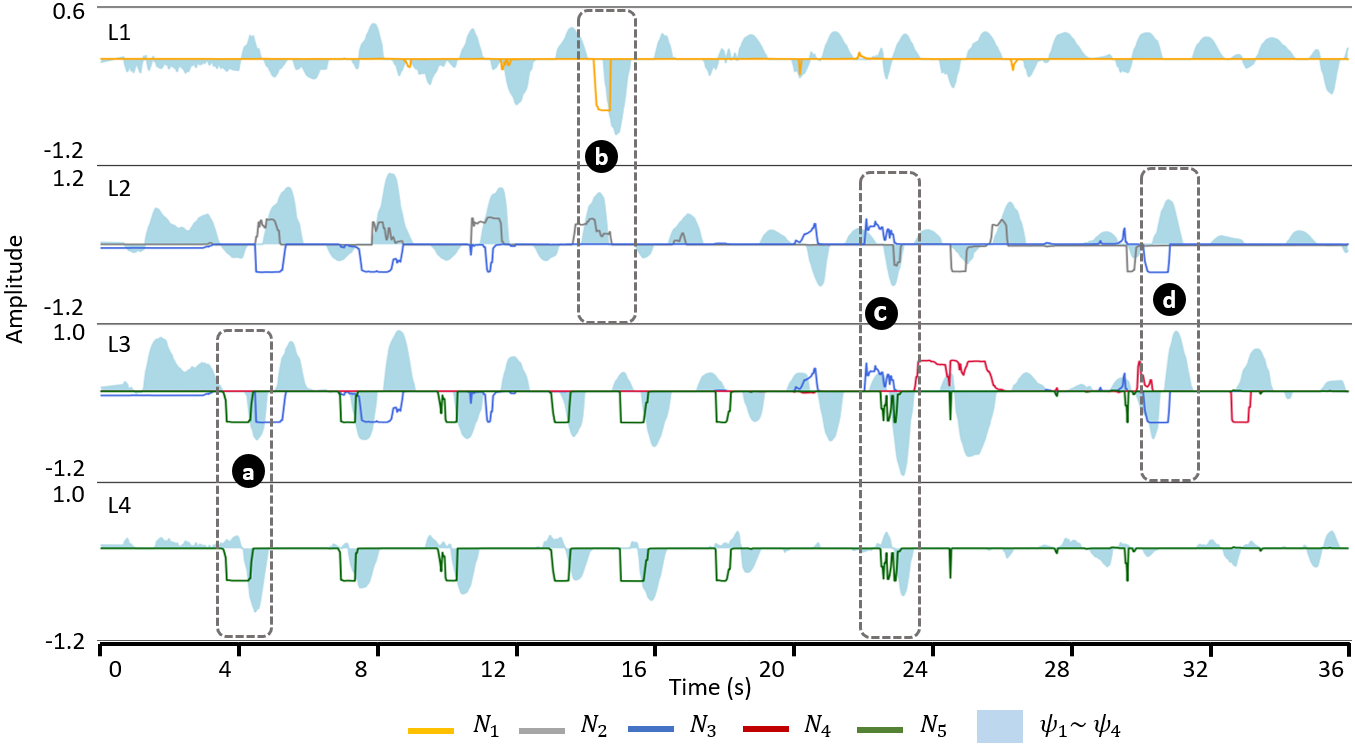}
    \caption{Recorded sensory input and CPG output of each body link of the soft snake robot controlled by the AF-local method in the goal-oriented escaping task.}
    \label{fig:escapingAFlocal}
\end{figure*} 

\begin{figure*}[h!]
    \centering
    \includegraphics[width=0.85\textwidth]{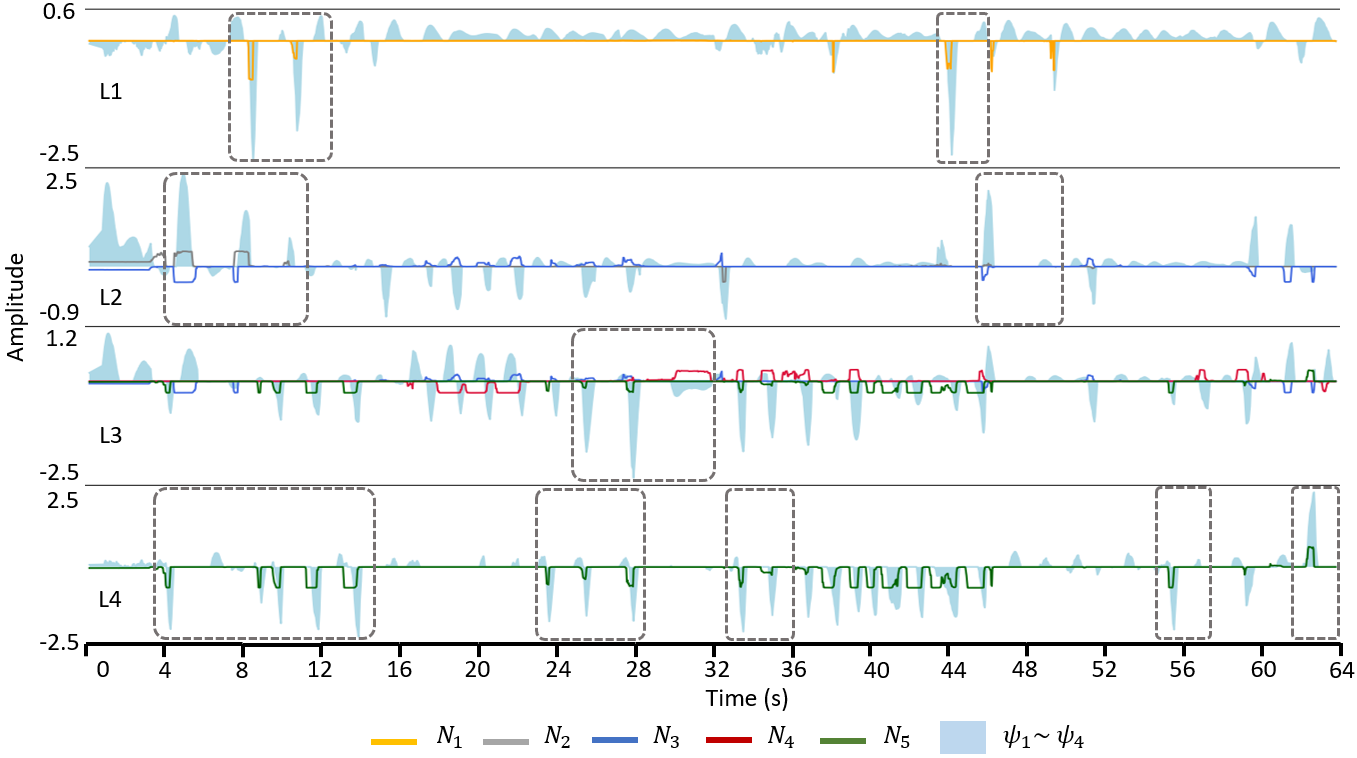}
    \caption{Recorded sensory input and CPG output of each body link of the soft snake robot controlled by the MPF-local method in the goal-oriented escaping task.}
    \label{fig:escapingMPFlocal}
\end{figure*}

In the real-world contact-aware locomotion scenario, we design an escaping task to distinguish the strength and weakness of the contact-aware controllers (listed in Section~\ref{sec:simtrain}). 

\noindent \textbf{Environment settings: }The escaping task is designed for the following principles:
\begin{itemize}
    \item The allocation of the obstacles should create a narrow passage for the snake robot, with more contact opportunities and a sharper tuning angle to test the overall capability of the controllers in escaping the jamming situations. In addition, the narrow space also limits the amplitude for regular body oscillation of the snake robot. 
    \item The obstacles should be allocated to obstruct the goal-reaching behavior. This is to test the coordination of the goal-reaching module (C1 controller) and contact reactive module (local reflexive or R2 method) in the compared controllers. 
    \item The allocation of the obstacles should include the situation where only the latter half links of the robot are stuck in the obstacles. This is for telling whether the controller relies mostly on its head steering to escape from the obstacles.
    \item The obstacles should be placed more densely in reality to test the generality of the compared controllers.
\end{itemize}

Based on the above principles, the obstacles in the escaping task are allocated as shown in Fig.~\ref{fig:videoscreenshot}. In the escaping task, the distance between every two obstacles ranged from $85$ mm to $150$ mm. The robot is initially bending to its left, and placed at a position where $4$ rigid bodies are in contact with the obstacles from different sides. The exit direction (left) of the obstacle region is intentionally set opposite to the goal direction (right). The distance between the exit of the obstacle region and the goal is $540$ mm, which is close to the length of the snake robot.

\noindent \textbf{Performance statistics: }According to the free oscillation tonic input property of coefficient $c$ in \cite[Appendix B-D]{xliu2023}, as the value of $c$ increases, it can increase the oscillation amplitude of the outputs of FOC-PPOC-CPG controller and therefore improve its sim-to-real adaptability in the locomotion tasks. However, the value of $c$ should not be larger since a higher free oscillation tonic input could decrease the goal-tracking accuracy. As a result, we separate the experiment into two groups with $c=0.2$ and $c=0.5$ respectively. For each value of $c$, we run five trials for each control method.\footnote{Performance videos for each method in the escaping task with different $c$ values are available at: \url{https://shorturl.at/huBR1}.} 

We record and compare the finishing time of the escaping task of each controller. As shown in Fig.~\ref{fig:snakeperformancecompare}, AF-local and AF-learning methods outperform the other methods in the escaping task in both speed and stability. The increase of $c$ from $0.2$ to $0.5$ does not significantly improve the performance of both AF methods. The main reason is that the sim-to-real adaptability of the AF methods is already good. MPF-learning method's average finishing time is shorter than MPF-local when $c=0.2$, but is less stable than MPF-local, with the task finishing time varying from 42 seconds to 120 seconds. When $c=0.5$, MPF-local method outperforms MPF-learning in both speed and stability. C1$+$ method cannot reach the goal in every trial when $c=0.2$. However, with the increase of $c$ to $0.5$, the adaptability of the C1$+$ controller is also improved so that it succeeds in a few of the trials. It is noted that, although MPF-learning converges to a lower reward level than C1$+$ method during the learning process (Fig.~\ref{fig:compare_result}), its adaptability to the harder unseen task (in sim-to-real) is better than C1$+$ method.  Generally, the results in Fig.~\ref{fig:snakeperformancecompare} further verify the advantages of AF feedback Matsuoka oscillator predicted by Remark~\ref{re:AFvsMPFvariable} and Proposition~\ref{prop:matsuokaduty}.

\begin{figure*}[ht!]
\centering
\includegraphics[width=0.85\textwidth]{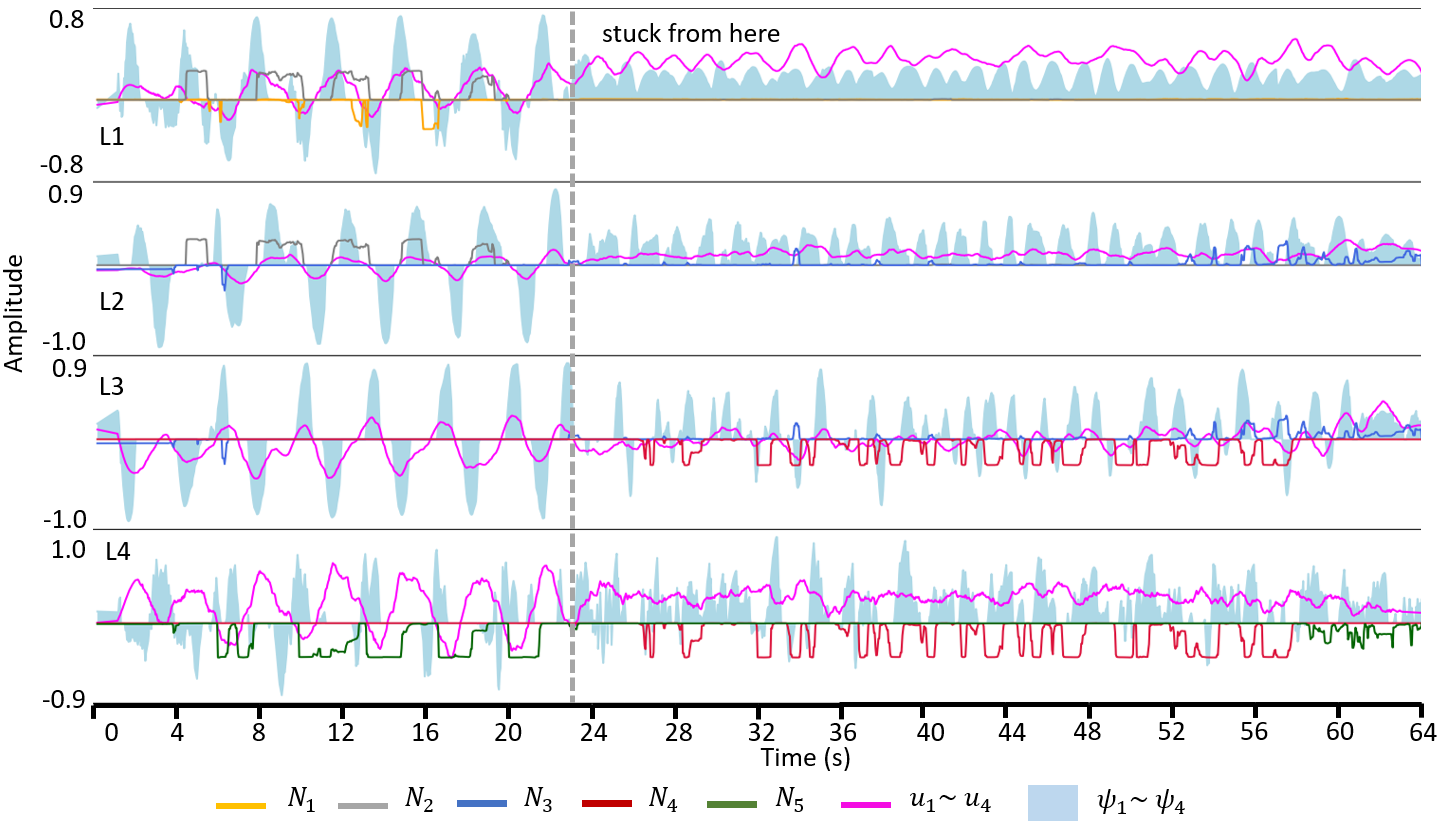}
\caption{Recorded sensor feedback control signals, tonic input signals and CPG outputs of the soft snake robot controlled by C1$+$ method in the goal-oriented escaping task.}
\label{fig:c1+}
\end{figure*}

\begin{figure*}[ht!]
\centering
\includegraphics[width=0.85\textwidth]{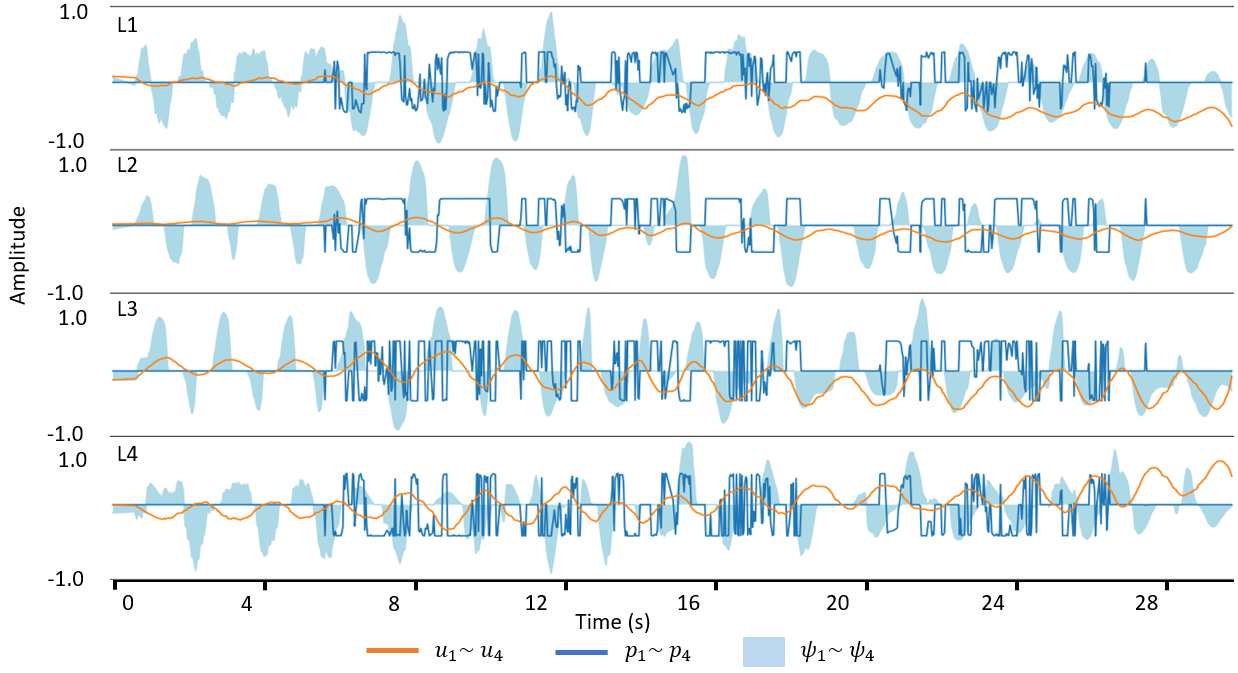}
\caption{Recorded sensor feedback control signals, tonic input signals and CPG outputs of the soft snake robot controlled by AF-learning method in the goal-oriented escaping task.}
\label{fig:plearning}
\end{figure*}

\begin{figure*}[ht!]
\centering
\includegraphics[width=0.85\textwidth]{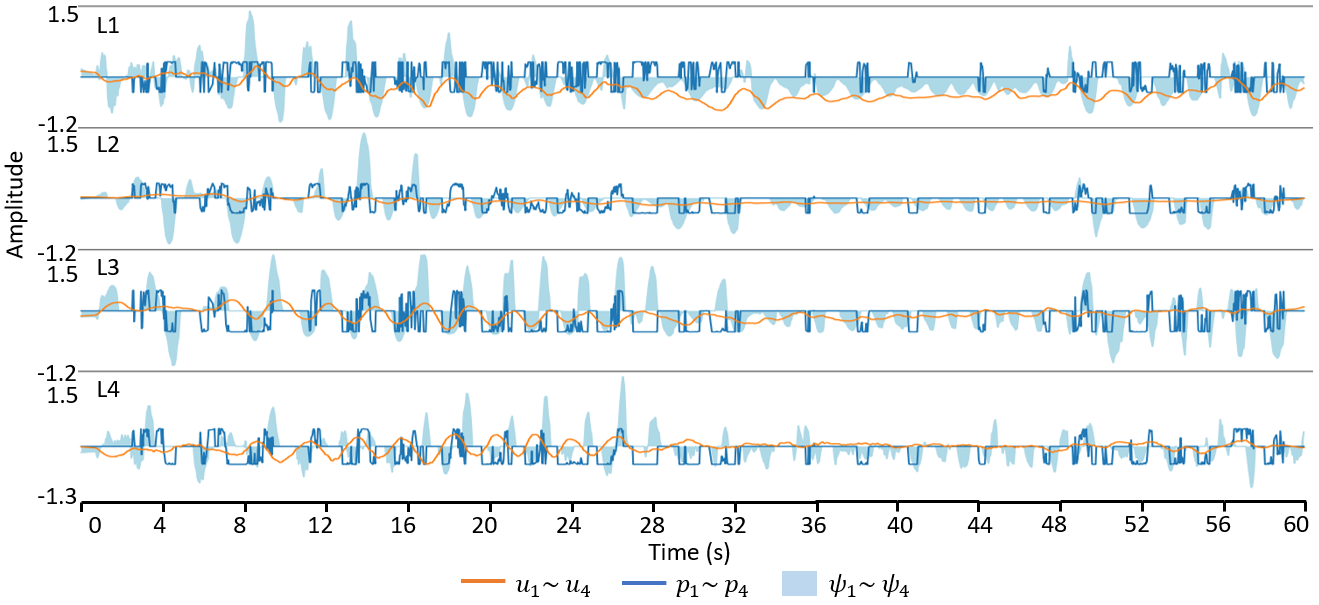}

\caption{Recorded sensor feedback control signals, tonic input signals and CPG outputs of the soft snake robot controlled by MPF-learning method in the goal-oriented escaping task.}
\label{fig:flearning}
\end{figure*}

\begin{figure*}[h!]
\centering
\includegraphics[width=0.8\textwidth]{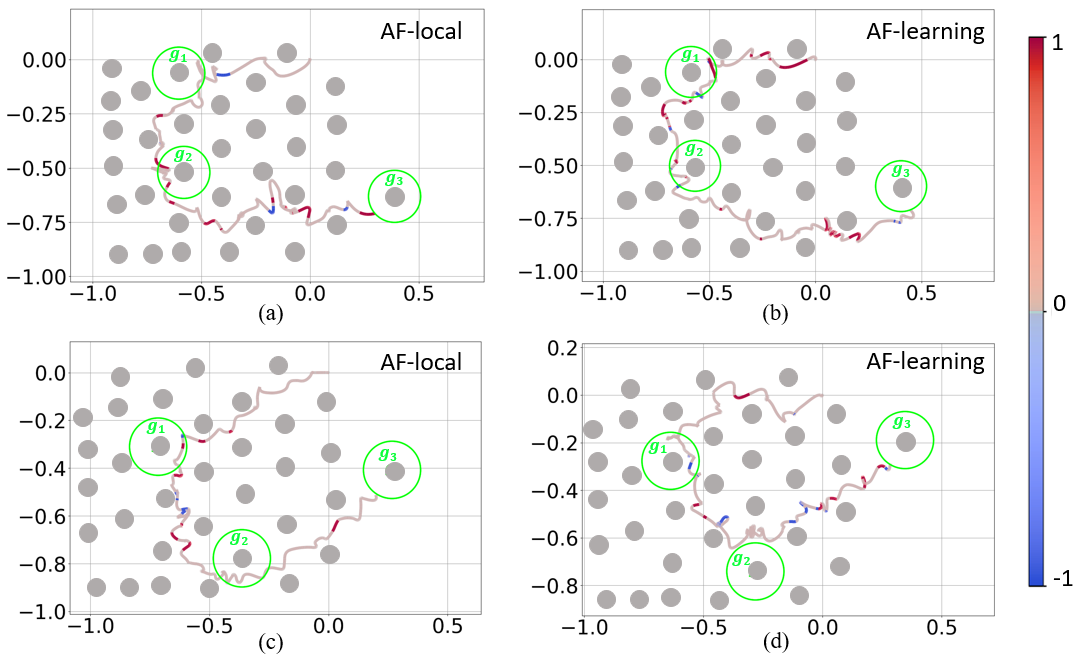}
\caption{Sample way-point trajectories followed by (a) AF-local controller in square trajectory, (b) AF-learning controller in square trajectory, (c) AF-local controller in triangle trajectory and (d) AF-learning controller in triangle trajectory. The distribution of reactive signals along the trajectory to the CPG-controlled actuators from head joint of the robot are visualized.}
\label{fig:diffshowcase}
\end{figure*}

\noindent \textbf{Case analysis: }
We can further compare the sample trajectories of contact feedback signals and control commands for different control methods to analyze the special features of AF-local and AF-learning methods (Fig.~\ref{fig:escapingAFlocal}, Fig.~\ref{fig:escapingMPFlocal}, Fig.~\ref{fig:flearning}, Fig.~\ref{fig:plearning} and Fig.~\ref{fig:c1+}). It is noted that in the joint space figures, the positive and negative values are related to the extensor and flexor of the CPG system, as well as the left and right of the snake body respectively.

First, we investigate the trajectory sample of the AF-local method in the escaping task on the basis of AF-local mechanisms illustrated by Fig.~\ref{fig:modifiedmatsuokareflexivemethod}. As shown in Fig.~\ref{fig:escapingAFlocal}, we highlight four time intervals of the trajectory that present typical local reflexive control in the AF-local controller (the robot's body postures before and after contacts at intervals (a)$\sim$(d) are captured by Fig.~\ref{fig:videoscreenshot}a$\sim$Fig.~\ref{fig:videoscreenshot}d). Here we select time intervals (a) and (c) for discussion. At time interval (a) of Fig.~\ref{fig:escapingAFlocal}, both CPG nodes at L3 and L4 are first influenced by the contact from the $N_5^f$ ($N_5 < 0$), so the flexors of CPG nodes in L3 and L4 are activated to open the right valves of L3 and L4, which results in both links bend to the left in Fig.~\ref{fig:videoscreenshot}a. Then L3's CPG output is influenced by $N_3^f$, which will deactivate L3's CPG flexor and activate L3's CPG extensor. At time interval (c) of Fig.~\ref{fig:escapingAFlocal}, the CPG node at L2 is influenced by the superposition of $N_3^e$ and $N_2^f$, and is supposed to activate its flexor to open the right valve of L2, which results in L2 bend leftward in Fig.~\ref{fig:videoscreenshot}c. Due to the whole snake robot's tendency to turn right toward the target position, the amplitude of L2's CPG output signal is smaller than expected. The CPG node at L3 is influenced by the superposition of $N_3^e$ and $N_5^f$, which also causes L3's flexor activated to bend to the left side. The CPG node at L4 is influenced by $N_5^f$, which activates L4's flexor and bends the L4 soft chamber to the left. From the above behavior of the CPG outputs, we can verify that the experiment results match the local reflexive mechanism illustrated in Fig.~\ref{fig:modifiedmatsuokareflexivemethod}.

% \begin{figure*}[ht!]
% \centering
% \includegraphics[width=0.85\textwidth]{figs/snake_escape_f_learning_analysis.png}

% \caption{Recorded sensor feedback control signals, tonic input signals and CPG outputs of the soft snake robot controlled by MPF-learning method in the goal-oriented escaping task.}
% \label{fig:flearning}
% \end{figure*}

% \begin{figure*}[ht!]
% \centering
% \includegraphics[width=0.85\textwidth]{figs/snake_escape_p_learning_analysis.png}
% \caption{Recorded sensor feedback control signals, tonic input signals and CPG outputs of the soft snake robot controlled by AF-learning method in the goal-oriented escaping task.}
% \label{fig:plearning}
% \end{figure*}

% \begin{figure*}[ht!]
% \centering
% \includegraphics[width=0.85\textwidth]{figs/snake_escape_sensor_free_analysis.png}
% \caption{Recorded sensor feedback control signals, tonic input signals and CPG outputs of the soft snake robot controlled by C1$+$ method in the goal-oriented escaping task.}
% \label{fig:c1+}
% \end{figure*}

Similarly, from the sampled trajectories of the MPF-local method in Fig.~\ref{fig:escapingMPFlocal}, we can conclude that the sensory inputs and the CPG outputs for all body links satisfy the local reflexive mechanism determined by Fig.~\ref{fig:modifiedmatsuokareflexivemethod}. However, when comparing Fig.~\ref{fig:escapingMPFlocal} to the AF-local behavior in Fig.~\ref{fig:escapingAFlocal}, the MPF-local controller produces significantly large overshoots even when the contact signals are small. The recovery delay is also more frequently observed in the trajectories of the MPF-method, such that the MPF-local controller always takes a longer time to recover to its goal-reaching oscillation after the contact signals disappear. These observations further verify Remark~\ref{re:AFvsMPFvariable} and its derivations, that the first order derivative term $\Dot{p_i}$ will seriously interfere with the control of MPF form CPG system when the contact feedback signals are densely emerging, and therefore hinder the performance of contact-aware locomotion.\footnote{The locomotion performance of the MPF-local method can be observed in videos \url{https://shorturl.at/fAGJN} and \url{https://shorturl.at/AORW9}.} 

The issue of the output wave response can also be observed in MPF-learning (Fig.~\ref{fig:flearning}). With more chaotic sensory feedback signals from the RL event-based controller R2, the CPG outputs also show disturbed behaviors, which significantly slow down the locomotion in the escaping task. It is worth noting that, due to the black-box property of the learning-based method, both AF-learning and MPF-learning methods send more complex sensory feedback signals to their CPG systems. However, we can still observe clear and coordinated oscillation in the sample of AF-learning trajectory in Fig.~\ref{fig:plearning}. This is also because AF series methods are free from the disturbances of the $\Dot{p_i}$ term.

% \begin{figure*}[ht!]
% \centering

% \subfloat[]{
% \includegraphics[width=0.8\textwidth]{figs/snake_escape_f_learning_analysis.png}
% \label{fig:flearning}}

% \subfloat[]{
% \includegraphics[width=0.8\textwidth]{figs/snake_escape_p_learning_analysis.png}
% \label{fig:plearning}}

% \subfloat[]{
% \includegraphics[width=0.8\textwidth]{figs/snake_escape_sensor_free_analysis.png}
% \label{fig:c1+}}

% \caption{(a) recorded sensor feedback control signals, tonic input signals and CPG outputs of the soft snake robot controlled MPF-learning method, (b) AF-learning method and (c) C1$+$ method in the goal-oriented escaping task.}
% \label{fig:analysis1}
% \end{figure*}

As shown in Fig.~\ref{fig:c1+}, the C1$+$ method fails to learn to react to the sensory inputs. When the target moving direction of the robot is blocked by the obstacles, the C1$+$ controller cannot pull the soft snake robot out from the jam and skirt the obstacles.  

In conclusion, the results and analyses in the escaping tasks show strong evidence of the advantage of AF-local and AF-learning controllers in the contact-aware locomotion of soft snake robots.

\subsubsection{General performance of AF series methods in difficult contact-aware locomotion tasks}

In this section, we throw the two methods with the best performance in the escaping task to a more complicated environment with multiple targets to traverse in a dense obstacle array, with a lot of detours.

Due to the randomness of contacts in such complicated goal-tracking tasks in the obstacle maze, it is harder for any methods to concentrate around certain trajectories. Without pre-planning the path, it is possible for the same control method to traverse the target goals in different paths. We show fast and slow paths for each task (square and triangle) in video ``square.mp4" and ``triangle.mp4" \footnote{The videos are available at \url{https://shorturl.at/huBR1}}. 

In Fig.~\ref{fig:diffshowcase}, we plot the recorded paths of AF-local and AF-learning methods traversing three targets allocated in square shape and triangle shape respectively in 2 dimension space. The color map on the paths shows the reactive commands of the L1 actuator sent by both control methods. Both methods have achieved decent performance in the harder tasks.

\section{Conclusion}
\label{sec:conclusion}
This paper establishes a novel framework for the contact-aware intelligent locomotion control of a soft snake robot. This framework is an organic integration of hardware design, feedback mechanism study through a bio-inspired CPG system, and 
implementation of sensory feedback control schemes. The proposed approaches are able to achieve promising performance in both simulation and real robots in several contact-aware locomotion tasks with densely allocated obstacles. Our novel method tackles jointly contact sensing and contact reacting controls in the contact-aware locomotion control of the soft snake robot. Our work brings inspiration for both the distributed reflexive method and learning-based control method and forms the basis to design and control of soft snake robots that can pass through environments with unpredictable and dense obstacles.

For future study, the contact module and design can be enhanced with the consideration of more advanced materials and structures to improve contact sensitivity and locomotion efficiency for more challenging environments (e.g. underwater contact or uneven and compliant terrains). The tactile information and the locomotion gait can also be enriched by increasing the number of body links of the snake robot. More investigation is also needed to understand the influence of couplings among primitive AF form feedback Matsuoka oscillators in the CPG network so that the variation of couplings can be utilized to improve the performance of the contact-aware locomotion controller. In addition, one limitation is that the proposed learning-based controller is mainly reactive and may not learn to leverage obstacles to aid the locomotion without trajectory planning. Such behavior may be achieved by combining depth visual information and tactile information of the obstacles to the deep reinforcement learning controller in the PPOC-CPG framework.

\appendices
\numberwithin{equation}{section}

\section{Data}
\label{app:data}
This section includes the parameter configuration of the Matsuoka CPG network and the hyper parameter setting of domain randomization for the experiment.

\begin{table}[h]
% \label{tab:cpgparam}
    % \renewcommand*{\arraystretch}{1.4}
    \centering
    \caption{Parameter Configuration of the Matsuoka CPG Net Controller for the Soft Snake Robot.}
    \label{tab:config}
    \scalebox{1}{
    \begin{tabular}{p{3cm}|p{2cm}|p{2cm}}
    %  \hline
    %  \multicolumn{2}{|c|}{Country List} \\
     \hline 
     \textbf{Parameters} & \textbf{Symbols} & \textbf{Values} \\  \hline
    %  \hline
        Amplitude ratio & $a_{\psi}$ & 2.0935\\ 
        $*$Self-inhibition weight & $b$ & \textbf{10.0355} \\
        $*$Discharge rate & $\tau_r$ & \textbf{0.7696} \\
        $*$Adaptation rate & $\tau_a$ & \textbf{1.7728} \\
        Period ratio & $K_f$ & 1.0\\[1ex] 
        % \hline
        \hline
        Mutual inhibition weights 
        %  & $\alpha_1$ & 2.1669 \\
        %  & $\alpha_2$ & 3.1948 \\
        %  & $\alpha_3$ & 5.3696 \\
        %  & $\alpha_4$ & 9.5222 \\[1ex]
         & $a_i$ & 4.6062 \\[1ex]
        % \hline
        \hline
        Coupling weights 
        % & $w_{12}$ & 4.1244 \\
        %  & $w_{23}$ & 5.0448 \\
        %  & $w_{34}$ & 8.5053 \\
        %  & $w_{21}$ & 8.3042 \\
        %  & $w_{32}$ & 8.1086 \\
        %  & $w_{43}$ & 6.2195 \\ [1ex]
         & $w_{ij}$ & 8.8669 \\
         & $w_{ji}$ & 0.7844 \\ [1ex]
     \hline
    \end{tabular}
}\end{table}

% \begin{table}[h]
%     \centering
%     \caption{Curriculum settings}
%     \scalebox{0.9}{
%     \begin{tabular}{c|c|c|c}
%         \hline
%         \textbf{Levels} & \textbf{Distance range ($m$)} & \textbf{Turning angles ($^{\circ}$)} & \textbf{Goal radius ($m$)} \\
%         \hline
%         1 & $1.2 \sim 1.5$ & $\boldsymbol{-}10 \sim 10$ & 0.5\\
%         2 & $1.2 \sim 1.5$ & $\boldsymbol{-}10 \sim 10$ & 0.4\\
%         3 & $1.2 \sim 1.5$ & $\boldsymbol{-}15 \sim 15$ & 0.3\\
%         4 & $1.2 \sim 1.5$ & $\boldsymbol{-}20 \sim 20$ & 0.25\\
%         5 & $1.2 \sim 1.5$ & $\boldsymbol{-}30 \sim 30$ & 0.2\\
%         6 & $1.0 \sim 1.5$ & $\boldsymbol{-}40 \sim 40$ & 0.18\\
%         7 & $1.0 \sim 1.5$ & $\boldsymbol{-}45 \sim 45$ & 0.15\\
%         8 & $1.0 \sim 1.5$ & $\boldsymbol{-}50 \sim 50$ & 0.12\\
%         9 & $0.9 \sim 1.5$ & $\boldsymbol{-}60 \sim 60$ & 0.09\\
%         10 & $0.9 \sim 1.5$ & $\boldsymbol{-}60 \sim 70$ & 0.06\\
%         11 & $0.9 \sim 1.5$ & $\boldsymbol{-}70 \sim 70$ & 0.05\\
%         12 & $0.8 \sim 1.5$ & $\boldsymbol{-}80 \sim 80$ & 0.05\\
%         \hline
%     \end{tabular}}
%     \label{tab:curriculum}
% \end{table}

\begin{table}[h]
    \centering
    \caption{Domain randomization parameters}
    \scalebox{1}{
    \begin{tabular}{c|c|c}
        \hline
        \textbf{Parameter} & \textbf{Low} & \textbf{High} \\
        \hline
        Ground friction coefficient & 0.1 & 1.5 \\
        Wheel friction coefficient& 0.05 & 0.10 \\
        Rigid body mass ($kg$)& 0.035 & 0.075 \\
        Tail mass ($kg$)& 0.065 & 0.085 \\
        Head mass ($kg$) & 0.075 & 0.125 \\
        Max link pressure ($psi$) & 5 & 12\\
        Gravity angle ($rad$) & -0.001 & 0.001\\
        \hline
    \end{tabular}}
    \label{tab:dr}
\end{table}

% \section{Experiment Figures}
% \label{app:expfig}

% \begin{figure*}[h!]
% \centering

% \subfloat[]{
% \includegraphics[width=0.85\textwidth]{figs/snake_escape_f_local_analysis.png}
% \label{fig:flocal}}

% \caption{(a) Recorded sensory input and CPG output of each body link of the soft snake robot controlled by the MPF-local method, and (d) C+ method in the goal-oriented escaping task.}
% \label{fig:analysis2}
% \end{figure*}

\section{Proof of Proposition ~\ref{prop:matsuokaduty}}
\label{app:theory}

\begin{proof}
    For simplicity we denote $A_i^q\triangleq A_{x_i^q}$ and $r_i^q\triangleq r_{x_i^q}$ for $q\in\{e, f\}$. Instead of looking into the relation between $u_i$ and $z_i$, we focus on the bias between the two states. 

According to the \textit{perfect entrainment assumption} \cite{matsuoka2011analysis} and \cite[Definition 1]{xliu2023}, let $u_i$ be resonant to $z_i$. We define the duty cycle of a wave function as $D(\cdot)$. Let the period of $z_i$ be $T = 2\pi$ (a different value of $T$ would not affect the result of the calculation), based on the Fourier expansion, the bias of $u_i$ can be expressed as
\begin{align}
\label{eq: biasu}
    \mbox{bias}(u_i) &= \frac{1}{T}\int_{-T/2}^{T/2} u_i(t) dt \\ \nonumber
    &= \frac{1}{2\pi}\int_{-\pi}^{\pi} u_i(t) dt\\ \nonumber
    &= 2 \frac{1}{2\pi}\int_{-\pi}^{\pi} u_i^e(t) dt - 1 = 2 D(u_i^e) - 1.
\end{align}

Because the bias terms of $x_i$ and $u_i$ are time-invariant, we can extract the bias component to form a new equation as follows
\begin{align}
\label{eq:biasraw}
\mbox{bias}(x_i) &= a\cdot \mbox{bias}(z_i) - b\cdot \mbox{bias}(y_i) + \mbox{bias}(u_i)\\ \nonumber
\mbox{bias}(y_i) &= \mbox{bias}(z_i) - \mbox{bias}(p_i).
\end{align}

Assume $x_i$ can be approximated by its main sinusoidal component and the period of both $x_i$ and $z_i$ is represented by $T$, we have 
\begin{align*}
    &\mbox{bias}(x_i) = \frac{1}{T}\int_{-T/2}^{T/2} x_i dt = \frac{1}{T}\int_{-T/2}^{T/2} (x_i^e -x_i^f) dt\\
    &\quad = \frac{1}{T}\int_{-T/2}^{T/2} A_i^e(\cos(\omega t) + r_i^e) - A_i^f(\cos(\omega t) + r_i^f) dt \\
    &\quad = A_i^e r_i^e - A_i^f r_i^f,
\end{align*}
and
\begin{align*}
    &\mbox{bias}(z_i) = \frac{1}{T}\int_{-T/2}^{T/2} z_i dt = \frac{1}{T}\int_{-T/2}^{T/2} (z_i^e - z_i^f) dt\\
    &\quad = \frac{1}{T}\int_{-T/2}^{T/2} (A_i^e(K(r_i^e)cos(\omega t) + L(r_i^e)) \\
    &\qquad \qquad - A_i^f(K(r_i^f)\cos(\omega t) + L(r_i^f))) dt \\
    &\quad= A_i^e (L(r_i^e) - \frac{1}{\pi}) - A_i^f (L(r_i^f) - \frac{1}{\pi}) + \frac{1}{\pi} (A_i^e - A_i^f).
\end{align*}
Apply Taylor expansion on $L(r)$ at $r=0$, we have
\[
    L(r) = \frac{1}{\pi} + \frac{r}{2} + o(r), r\in(-1, 1).
\]
Then we have
\begin{align}
\label{eq:bzbx}
    \mbox{bias}(z_i)&= \frac{1}{2}A_i^e r_i^e - \frac{1}{2}A_i^f r_i^f +\frac{1}{\pi}(A_i^e-A_i^f)\\ \nonumber
    &= \frac{1}{2}\mbox{bias}(x_i) +\frac{1}{\pi}(A_i^e-A_i^f).
\end{align}

According to \cite{matsuoka2011analysis}, the amplitude $A_i^q$ (for $q\in\{e,f\}$) has the form
\[
    A_i^q = \frac{\mbox{bias}(u_i^q)+c}{r_i^q+(a+b)L(r_i^q)}.
\]
When the system is harmonic, according to \cite[(30)]{matsuoka2011analysis}, we have
\[
    r_i^e = r_i^f = K^{-1}(K_n),
\]
such that
\begin{align}
    A_i^e - A_i^f &= \frac{\mbox{bias}(u_i^e)-\mbox{bias}(u_i^f)}{ K^{-1}(K_n)+(a+b)L( K^{-1}(K_n))}\\
    &\approx \frac{\mbox{bias}(u_i)}{ 2K_n-1+\frac{2}{\pi}(a+b)\sin^{-1}(K_n)}.
\end{align}

Let $m = \frac{1}{\pi}\frac{1}{2K_n-1+\frac{2}{\pi}(a+b)\sin^{-1}(K_n)}$, \eqref{eq:bzbx} can be rewritten as
\begin{align}
\label{eq:bzbxbu}
    \mbox{bias}(z_i) = \frac{1}{2}\mbox{bias}(x_i) + m\mbox{bias}(u_i).
\end{align}
Substitute $\mbox{bias}(x_i)$ in \eqref{eq:biasraw} with \eqref{eq:bzbxbu}, we can obtain the pure relation between $\mbox{bias}(z_i)$ and $\mbox{bias}(u_i)$, $\mbox{bias}(p_i)$ as

\begin{align*}
    \mbox{bias}(z_i) = \frac{1+2m}{b-a+2}\mbox{bias}(u_i)+ \frac{b}{b-a+2}\mbox{bias}(p_i).
\end{align*}
\end{proof}

\section{The Reason of Using Inhibiting Sensory Feedback Input in the AF form Matsuoka Oscillator}
\label{app:remark}

In the AF form Matsuoka oscillator, if the sensory feedback coefficients $p_i^e, p_i^f$ are activating, then by changing the sign of $p_i$ term in Eq.~\eqref{eq:odematsuokaP}, we have
\begin{align}
\label{eq:wrongodematsuokaP}
    &\tau_r \tau_a \frac{d^2}{dt^2} x_{{\cal_F}_i} + (\tau_r + \tau_a - \tau_a a K(r_x)) \frac{d}{dt} x_{{\cal_F}_i} \\ \nonumber
    &+ ((b-a) K(r_x) + 1)x_{{\cal_F}_i} = \tau_a\frac{d}{dt}u_{{\cal_F}_i} + u_{{\cal_F}_i} - b p_i.
\end{align}
Let $m_1 = \tau_r \tau_a$, $m_2 = \tau_r + \tau_a - \tau_a a K(r_x)$ and $m_3 = (b-a) K(r_x) + 1$, $\omega_0^2 = \frac{m_3}{m_1}$. We have Eq.~\eqref{eq:wrongodematsuokaP} simplified to
\begin{align}
    m_1 \frac{d^2}{dt^2} x_{{\cal_F}_i} + m_2 \frac{d}{dt} x_{{\cal_F}_i} + m_3 x_{{\cal_F}_i} = \tau_a\frac{d}{dt}u_{{\cal_F}_i} + u_{{\cal_F}_i} - b p_i.
\end{align}
According to \cite[Eq. (B.15), Eq. (B.16)]{xliu2023}, 
\[
    u_{{\cal_F}_i} \approx A_u \cos{(\omega t)} - \frac{1}{2}, \text{ where } A_u = \frac{1}{2}\frac{e^A - 1}{e^A + 1}.
\]
Therefore we have
\begin{align}
\label{eq:wrongodematsuokaPsimp}
    & m_1 \frac{d^2}{dt^2} x_{{\cal_F}_i} + m_2 \frac{d}{dt} x_{{\cal_F}_i} + m_3 x_{{\cal_F}_i} \\ \nonumber
    &= -\tau_a \omega A_u \sin{(\omega t)} 
    + A_u \cos{(\omega t)} - \frac{1}{2} - b p_i. 
\end{align}

During a contact segment, assume $p_i \approx 1$ and $p_i$ is a constant in this segment (e.g. the first half period of a square wave). We consider the solutions of Eq.~\eqref{eq:wrongodematsuokaPsimp} in the following two scenarios:
\begin{itemize}
    \item When $K(r_x) < \frac{\tau_r + \tau_a}{\tau_a a}$, the particular solution is
    \begin{align}
        x^*(t) &= -\frac{\tau_a \omega A_u}{\sqrt{m_1^2(\omega_0^2-\omega^2)^2 + m_2^2\omega^2}} \cos{(\omega t + \theta_1)} \\ \nonumber
        &+ \frac{A_u}{\sqrt{m_1^2(\omega_0^2-\omega^2)^2 + m_2^2\omega^2}}  \cos{(\omega t + \theta_2)} \\ \nonumber
        &-\frac{\frac{1}{2} + b}{m_3},
    \end{align}
    and the general solution is 
    \begin{align*}
        x(t) = c_1 \exp{(\lambda_1 t)} + c_2 \exp{(\lambda_2 t)} + x^*(t) \approx x^*(t),
    \end{align*}
    where $\lambda_1, \lambda_2 < 0$ are the eigenvalues of Eq.~\eqref{eq:wrongodematsuokaPsimp}.
    
    \item When $K(r_x) = K_n = \frac{\tau_r + \tau_a}{\tau_a a}$ and $\omega \neq \omega_0$, the system becomes harmonic. The particular solution is
    \begin{align}
        x^*(t) &= -\frac{\tau_a \omega A_u}{m_1(\omega_0^2 - \omega^2)} \sin{(\omega t)} \\ \nonumber
        &+ \frac{A_u}{m_1(\omega_0^2 - \omega^2)}  \cos{(\omega t)} \\ \nonumber
        &-\frac{\frac{1}{2} + b}{m_3},
    \end{align}    
    and the general solution is 
    \begin{align*}
        x(t) = c_1 \cos{(\omega_0 t)} + c_2 \sin{(\omega_0 t)} + x^*(t),
    \end{align*}
    where the parameters $c_1, c_2$ are related to the initial condition of Eq.~\eqref{eq:wrongodematsuokaPsimp}. 
\end{itemize}

When the snake robot is jammed by the obstacles, the oscillation frequency $\omega$ will decrease, and therefore leads to the increase of $\omega_0^2 - \omega^2$. In this case, as long as $\omega$ is small enough, in both harmonic and non-harmonic situations the solution of Eq.~\eqref{eq:wrongodematsuokaPsimp} will be consistently negative. In the Matsuoka oscillator, a consistently negative membrane potential will cause $z_i = 0$. Then the output of the Matsuoka oscillator becomes zero. Therefore the CPG node in contact will always stop oscillating for a certain period of time, which extend the jamming period and harm the locomotion performance in contact-aware scenario.

\bibliographystyle{IEEEtran}
\bibliography{refs.bib}

\begin{thebibliography}{10}
\providecommand{\url}[1]{#1}
\csname url@samestyle\endcsname
\providecommand{\newblock}{\relax}
\providecommand{\bibinfo}[2]{#2}
\providecommand{\BIBentrySTDinterwordspacing}{\spaceskip=0pt\relax}
\providecommand{\BIBentryALTinterwordstretchfactor}{4}
\providecommand{\BIBentryALTinterwordspacing}{\spaceskip=\fontdimen2\font plus
\BIBentryALTinterwordstretchfactor\fontdimen3\font minus
  \fontdimen4\font\relax}
\providecommand{\BIBforeignlanguage}[2]{{%
\expandafter\ifx\csname l@#1\endcsname\relax
\typeout{** WARNING: IEEEtran.bst: No hyphenation pattern has been}%
\typeout{** loaded for the language `#1'. Using the pattern for}%
\typeout{** the default language instead.}%
\else
\language=\csname l@#1\endcsname
\fi
#2}}
\providecommand{\BIBdecl}{\relax}
\BIBdecl

\bibitem{hawkes2017soft}
E.~W. Hawkes, L.~H. Blumenschein, J.~D. Greer, and A.~M. Okamura, ``A soft
  robot that navigates its environment through growth,'' \emph{Science
  Robotics}, vol.~2, no.~8, 2017.

\bibitem{majidi2014soft}
C.~Majidi, ``Soft robotics: {A} perspective{--}{C}urrent trends and prospects
  for the future,'' \emph{Soft Robotics}, vol.~1, no.~1, pp. 5--11, 2014.

\bibitem{wang2017cable}
H.~Wang, R.~Zhang, W.~Chen, X.~Wang, and R.~Pfeifer, ``A cable-driven soft
  robot surgical system for cardiothoracic endoscopic surgery: preclinical
  tests in animals,'' \emph{Surgical endoscopy}, vol.~31, no.~8, pp.
  3152--3158, 2017.

\bibitem{7838565}
F.~{Sanfilippo}, J.~{Azpiazu}, G.~{Marafioti}, A.~A. {Transeth},
  {\O}.~{Stavdahl}, and P.~{Liljebäck}, ``A review on perception-driven
  obstacle-aided locomotion for snake robots,'' in \emph{2016 14th
  International Conference on Control, Automation, Robotics and Vision
  (ICARCV)}, 2016, pp. 1--7.

\bibitem{Kano2017}
\BIBentryALTinterwordspacing
T.~Kano, R.~Yoshizawa, and A.~Ishiguro, ``Tegotae-based decentralised control
  scheme for autonomous gait transition of snake-like robots,''
  \emph{Bioinspiration \& Biomimetics}, vol.~12, no.~4, p. 046009, aug 2017.
  [Online]. Available: \url{https://dx.doi.org/10.1088/1748-3190/aa7725}
\BIBentrySTDinterwordspacing

\bibitem{2023classicfeedback}
Y.~Fukuoka, K.~Otaka, R.~Takeuchi, K.~Shigemori, and K.~Inoue, ``Mechanical
  designs for field undulatory locomotion by a wheeled snake-like robot with
  decoupled neural oscillators,'' \emph{IEEE Transactions on Robotics},
  vol.~39, no.~2, pp. 959--977, 2023.

\bibitem{Transeth2007}
A.~A. Transeth, P.~Liljeback, and K.~Y. Pettersen, ``Snake robot obstacle aided
  locomotion: An experimental validation of a non-smooth modeling approach,''
  in \emph{2007 IEEE/RSJ International Conference on Intelligent Robots and
  Systems}, 2007, pp. 2582--2589.

\bibitem{transeth2008snake}
A.~A. Transeth, R.~I. Leine, C.~Glocker, K.~Y. Pettersen, and P.~Liljeb{\"a}ck,
  ``Snake robot obstacle-aided locomotion: Modeling, simulations, and
  experiments,'' \emph{IEEE Transactions on Robotics}, vol.~24, no.~1, pp.
  88--104, 2008.

\bibitem{liljeback2010hybrid}
P.~Liljeback, K.~Y. Pettersen, {\O}.~Stavdahl, and J.~T. Gravdahl, ``Hybrid
  modelling and control of obstacle-aided snake robot locomotion,'' \emph{IEEE
  Transactions on Robotics}, vol.~26, no.~5, pp. 781--799, 2010.

\bibitem{liljeback2011experimental}
------, ``Experimental investigation of obstacle-aided locomotion with a snake
  robot,'' \emph{IEEE Transactions on Robotics}, vol.~27, no.~4, pp. 792--800,
  2011.

\bibitem{GRAVDAHL2022247}
\BIBentryALTinterwordspacing
I.~Gravdahl, Øyvind Stavdahl, A.~Koushan, J.~Løwer, and K.~Y. Pettersen,
  ``Modeling for hybrid obstacle-aided locomotion (hoal) of snake robots,''
  \emph{IFAC-PapersOnLine}, vol.~55, no.~20, pp. 247--252, 2022, 10th Vienna
  International Conference on Mathematical Modelling MATHMOD 2022. [Online].
  Available:
  \url{https://www.sciencedirect.com/science/article/pii/S2405896322012940}
\BIBentrySTDinterwordspacing

\bibitem{kano2012local}
T.~Kano, T.~Sato, R.~Kobayashi, and A.~Ishiguro, ``Local reflexive mechanisms
  essential for snakes' scaffold-based locomotion,'' \emph{Bioinspiration \&
  biomimetics}, vol.~7, no.~4, p. 046008, 2012.

\bibitem{kano2013}
T.~Kano and A.~Ishiguro, ``Obstacles are beneficial to me! scaffold-based
  locomotion of a snake-like robot using decentralized control,'' in \emph{2013
  IEEE/RSJ International Conference on Intelligent Robots and Systems}, 2013,
  pp. 3273--3278.

\bibitem{10.1093/icb/icaa014}
\BIBentryALTinterwordspacing
------, ``{Decoding Decentralized Control Mechanism Underlying Adaptive and
  Versatile Locomotion of Snakes},'' \emph{Integrative and Comparative
  Biology}, vol.~60, no.~1, pp. 232--247, 03 2020. [Online]. Available:
  \url{https://doi.org/10.1093/icb/icaa014}
\BIBentrySTDinterwordspacing

\bibitem{kano2019}
T.~Kano, N.~Matsui, and A.~Ishiguro, ``3d movement of snake robot driven by
  tegotae-based control,'' in \emph{Biomimetic and Biohybrid Systems},
  U.~Martinez-Hernandez, V.~Vouloutsi, A.~Mura, M.~Mangan, M.~Asada, T.~J.
  Prescott, and P.~F. Verschure, Eds.\hskip 1em plus 0.5em minus 0.4em\relax
  Cham: Springer International Publishing, 2019, pp. 346--350.

\bibitem{ijspeert2021science}
\BIBentryALTinterwordspacing
R.~Thandiackal, K.~Melo, L.~Paez, J.~Herault, T.~Kano, K.~Akiyama, F.~Boyer,
  D.~Ryczko, A.~Ishiguro, and A.~J. Ijspeert, ``Emergence of robust
  self-organized undulatory swimming based on local hydrodynamic force
  sensing,'' \emph{Science Robotics}, vol.~6, no.~57, p. eabf6354, 2021.
  [Online]. Available:
  \url{https://www.science.org/doi/abs/10.1126/scirobotics.abf6354}
\BIBentrySTDinterwordspacing

\bibitem{min2019softcon}
S.~Min, J.~Won, S.~Lee, J.~Park, and J.~Lee, ``Softcon: simulation and control
  of soft-bodied animals with biomimetic actuators,'' \emph{ACM Transactions on
  Graphics (TOG)}, vol.~38, no.~6, pp. 1--12, 2019.

\bibitem{wang2016touchsensor}
\BIBentryALTinterwordspacing
H.~Wang, G.~De~Boer, J.~Kow, A.~Alazmani, M.~Ghajari, R.~Hewson, and P.~Culmer,
  ``Design methodology for magnetic field-based soft tri-axis tactile
  sensors,'' \emph{Sensors}, vol.~16, no.~9, 2016. [Online]. Available:
  \url{https://www.mdpi.com/1424-8220/16/9/1356}
\BIBentrySTDinterwordspacing

\bibitem{crowe2016sensilla}
J.~M. Crowe-Riddell, E.~P. Snelling, A.~P. Watson, A.~K. Suh, J.~C. Partridge,
  and K.~L. Sanders, ``The evolution of scale sensilla in the transition from
  land to sea in elapid snakes,'' \emph{Open biology}, vol.~6, no.~6, p.
  160054, 2016.

\bibitem{xliu2023}
X.~Liu, C.~D. Onal, and J.~Fu, ``Reinforcement learning of cpg-regulated
  locomotion controller for a soft snake robot,'' \emph{IEEE Transactions on
  Robotics}, pp. 1--20, 2023.

\bibitem{renato2019}
R.~Gasoto, M.~Macklin, X.~Liu, Y.~Sun, K.~Erleben, C.~Onal, and J.~Fu, ``A
  validated physical model for real-time simulation of soft robotic snakes,''
  in \emph{IEEE International Conference on Robotics and Automation}, 2019.

\bibitem{xliu2020}
X.~{Liu}, R.~{Gasoto}, Z.~{Jiang}, C.~{Onal}, and J.~{Fu}, ``Learning to
  locomote with artificial neural-network and cpg-based control in a soft snake
  robot,'' in \emph{International Conference on Intelligent Robots and
  Systems}, 2020.

\bibitem{luo2017toward}
M.~Luo, E.~H. Skorina, W.~Tao, F.~Chen, S.~Ozel, Y.~Sun, and C.~D. Onal,
  ``Toward modular soft robotics: Proprioceptive curvature sensing and
  sliding-mode control of soft bidirectional bending modules,'' \emph{Soft
  robotics}, vol.~4, no.~2, pp. 117--125, 2017.

\bibitem{matsuoka1985sustained}
K.~Matsuoka, ``Sustained oscillations generated by mutually inhibiting neurons
  with adaptation,'' \emph{Biological cybernetics}, vol.~56, no. 5-6, pp.
  367--376, 1985.

\bibitem{matsuoka2011analysis}
------, ``Analysis of a neural oscillator,'' \emph{Biological Cybernetics},
  vol. 104, no. 4-5, pp. 297--304, 2011.

\bibitem{sato2011}
Y.~D. Sato, K.~Nakada, and K.~Matsuoka, ``Singular perturbation approach with
  matsuoka oscillator and synchronization phenomena,'' in \emph{Artificial
  Neural Networks and Machine Learning -- ICANN 2011}, T.~Honkela, W.~Duch,
  M.~Girolami, and S.~Kaski, Eds.\hskip 1em plus 0.5em minus 0.4em\relax
  Berlin, Heidelberg: Springer Berlin Heidelberg, 2011, pp. 269--276.

\bibitem{khatib1986real}
O.~Khatib, ``Real-time obstacle avoidance for manipulators and mobile robots,''
  pp. 396--404, 1986.

\bibitem{park2001obstacle}
M.~G. Park, J.~H. Jeon, and M.~C. Lee, ``Obstacle avoidance for mobile robots
  using artificial potential field approach with simulated annealing,'' vol.~3,
  pp. 1530--1535, 2001.

\bibitem{dong2012strategies}
J.~Dong, X.~Zhang, and X.~Jia, ``Strategies of pursuit-evasion game based on
  improved potential field and differential game theory for mobile robots,''
  pp. 1452--1456, 2012.

\bibitem{GARRIDOJURADO20142280}
S.~Garrido-Jurado, R.~Muñoz-Salinas, F.~Madrid-Cuevas, and M.~Marín-Jiménez,
  ``Automatic generation and detection of highly reliable fiducial markers
  under occlusion,'' \emph{Pattern Recognition}, vol.~47, no.~6, pp.
  2280--2292, 2014.

\bibitem{tobin2017domain}
J.~Tobin, R.~Fong, A.~Ray, J.~Schneider, W.~Zaremba, and P.~Abbeel, ``Domain
  randomization for transferring deep neural networks from simulation to the
  real world,'' \emph{International Conference on Intelligent Robots and
  Systems}, pp. 23--30, 2017.

\end{thebibliography}
\vspace{4em}

% \appendix

\end{document}